\documentclass[journal,12pt,onecolumn,draftclsnofoot]{IEEEtran}
\usepackage{diagbox}
\usepackage{amsmath}
\usepackage{cases}
\usepackage{tikz}
\usepackage{graphicx}
\usepackage[caption=false,font=normalsize,labelfont=sf,textfont=sf]{subfig}
\usepackage{epsfig}
\usepackage{mathrsfs}
\usepackage{amssymb}
\usepackage{picinpar}
\usepackage{wrapfig}
\usepackage{savesym}
\savesymbol{AND}
\usepackage{algorithm}
\usepackage{algorithmic}
\usepackage{multirow}
\usepackage{amsmath}
\usepackage{booktabs}
\usepackage{picinpar}
\usepackage{amsthm}
\usepackage{bm}
\usepackage{color}
\usepackage{bookmark}
\definecolor{Annotation}{rgb}{255,0,0}
\usepackage{threeparttable}
\usepackage{tabularx}
\usepackage{epstopdf}
\usepackage{epsfig}
\usepackage{setspace}

\usepackage{graphicx}

\begin{document}			
	\title{A Safe Screening Rule with Bi-level Optimization of $\nu$ Support Vector Machine}		
	
	\author{Zhiji Yang, Wanyi Chen, 
		Huan Zhang, Yitian Xu, Lei Shi, Jianhua Zhao
		
		\thanks{This work was supported in part by the National Natural Science Foundation under Grant 62006206, 12161089, 11931015, 12271471, the Scientific Research Fund Project of Yunnan Provincial Department of Science and Technology under Grant 202001AU070064, and the Yunnan Provincial Department of Education Science Research Fund Project under Grant 2023Y0658.}
		\thanks{Zhiji Yang, Wanyi Chen, Huan Zhang, Lei Shi, and Jianhua Zhao are with the College of Statistics and Mathematics, Yunnan University of Finance and Economics, Yunnan 650221, China (e-mail: yangzhiji@ynufe.edu.cn).}
		\thanks{Yitian Xu is with the College of Science, China
			Agricultural University,  Beijing 100083, China.}}

	\maketitle	
	\begin{abstract}
		Support vector machine (SVM) has achieved many successes in machine learning, especially for a small sample problem. As a famous extension of the traditional SVM, the $\nu$ support vector machine ($\nu$-SVM) has shown outstanding performance due to its great model interpretability. However, it still faces challenges in training overhead for large-scale problems. To address this issue, we propose a safe screening rule with bi-level optimization for $\nu$-SVM (SRBO-$\nu$-SVM) which can screen out inactive samples before training and reduce the computational cost without sacrificing the prediction accuracy. Our SRBO-$\nu$-SVM is strictly deduced by integrating the Karush-Kuhn-Tucker (KKT) conditions, the variational inequalities of convex problems and the $\nu$-property. Furthermore, we develop an efficient dual coordinate descent method (DCDM) to further improve computational speed. Finally, a unified framework for SRBO is proposed to accelerate many SVM-type models, and it is successfully applied to one-class SVM. Experimental results on 6 artificial data sets and 30 benchmark data sets have verified the effectiveness and safety of our proposed methods in supervised and unsupervised tasks.
	\end{abstract}
	
	
	\IEEEpeerreviewmaketitle
	
	\section{Introduction}
	Support vector machine (SVM) is one of the most successful classification methods in the field of machine learning \cite{Steinwart3}. Its basic idea is to construct two parallel hyperplanes to separate two classes of instances and maximize the distance between the hyperplanes. The original version of support vector machine was proposed in 1995 by Vapnik et al. \cite{Vapnik}. It is often called $C$-SVM since there is a trade-off parameter $C$ that gives the penalty for incorrectly separated samples. When $C$ is large, the model tends to separate all training samples as correctly as possible. Otherwise, the model is more inclined to implement the maximal-margin principle of positive and negative hyperplanes. 
	However, selection of the parameter $C$ often lacks theoretical guidance and is generally selected in a wide range, such as the interval [$2^{-3}$, $2^{8}$], by cross-validation with grid search. It is usually difficult to select the optimal parameter from a wide range by a grid search. In addition, this approach needs to repeatedly solve many complete optimization problems, which is faced with the problem of computational overhead when data sets are large.
	
	The classic $\nu$ support vector machine ($\nu$-SVM) \cite{scholkopfNewSupportVector2000} is a successful modification of $C$-SVM. More importantly, in a skillful manner, a more interpretable parameter $\nu$ is introduced to replace the original parameter $C$. The parameter $\nu$ controls the bounds in the proportion of the support vectors and the error bound. It ranges from 0 to 1, and provides great convenience for parameter selection. Furthermore, some improved models based on $\nu$-SVM are proposed, such as the parametric insensitive / marginal model \cite{haoNewSupportVector2010}. Although its prediction accuracy is improved to some extent, more parameters are added, making parameter selection more expensive. In addition, an unsupervised version of $\nu$-SVM is presented, named the one-class support vector machine (OC-SVM) \cite{scholkopfEstimatingSupportHighDimensional2001}. It learns a hyperplane and gets a region which contains all training instances. If a test instance lies outside the region, it is declared an outlier. 
	The performance of OC-SVM has been fully analyzed and has a high prediction accuracy. It is widely used to deal with anomaly detection \cite{yinFaultDetectionBased2014,chalapathyAnomalyDetectionUsing2019}, user recommendation \cite{yajimaOneClassSupportVector2006a}, Handwritten Signature Verification System (HSVS) \cite{guerbai2015effective}. \cite{kauffmann2020towards} combines OC-SVM with deep Taylor decomposition (DTD) to propose OC-DTD, which is applicable to many common distance-based kernel functions and outperforms baselines such as sensitivity analysis. A variant of OC-SVM, called support vector data description (SVDD) \cite{tax_support_2004}, finds a hypersphere rather than a hyperplane to improve prediction performance, and has recently been extended to deep learning \cite{2018Deep}. Furthermore, \cite{xing2023contrastive} proposes contrastive deep SVDD (CDSVDD) to improve the performance of SVDD in processing large-scale data sets.
	
	Although $\nu$-SVM type models achieve admirable prediction performance in many applications, high time complexity is an obstacle in dealing with large-scale problems. Some research is devoted to saving computational cost or improving parameter selection. For example, in \cite{changTrainingSupportVector2001}, the authors investigate the relationship between $\nu$ -SVM and $C$ -SVM and propose a decomposition method for $\nu$-SVM to improve its efficiency. \cite{steinwartOptimalParameterChoice2003a} has proved that the parameter $\nu$ is a close upper estimate of twice the optimal Bayes risk, and this result can be used to improve the standard cross-validation for $\nu$-SVM. \cite{changLIBSVMLibrarySupport2011} develops a popular library for SVM (LIBSVM), which considers a sequential minimal optimization (SMO) type decomposition method \cite{platt1998fast} to solve optimization problems. \cite{hsieh_dual_2008} proposes a dual coordinate descent method (DCDM) for large-scale linear SVM starting from the dual problem, which has proven very suitable for use in large-scale sparse problems. To further improve efficiency, \cite{wenthundersvm18} has presented an effective open source ThunderSVM software toolkit that takes advantage of the high performance of graphics processing units (GPUs) and multi-core CPUs.
	
	Although these methods greatly improve the efficiency of traditional SVMs, the sparsity of SVM-type models is ignored. That is, these efficient algorithms above do not consider that the hyperplanes of the models can be completely determined by only a few support vectors and most of the samples are not necessary. Based on the above considerations, we focus on the sample screening (often called sample selection) method for the $\nu$-SVM type model in this paper.
	
	Many traditional sample screening methods \cite{xiajiantaoFastTrainingAlgorithm2003} cannot guarantee safety. That is, screened samples may directly affect the prediction accuracy of the final classifier. In recent years, a novel approach called `safe screening' has been presented. One of them directly achieves the safe screening rule (SSR) based on optimality conditions from the optimization problem of the classification model. It can guarantee to achieve the same solution as the original models. The safe screening method has been widely used to handle large-scale data for sparse models, such as LASSO \cite{wangLassoScreeningRules2014}, sparse logistic regression \cite{wangSafeScreeningRule2013,panSafeFeatureElimination2021}, and SVM \cite{pmlr-v28-ogawa13b} \cite{dantasSafeScreeningSparse2021}. In particular, SSR to solve the dual problem of SVM (the dual problem of SVM via variational inequalities, DVI-SVM) \cite{pmlr-v32-wangd14} has achieved remarkable results in $C$-SVM. This paper embeds SSR in the training process and safely identifies non-support vectors before solving the optimization problem, which could greatly reduce computational cost and memory. Furthermore, this method has been extended to several modified SVMs \cite{caoMultivariableEstimationbasedSafe2020,yangSafeAccelerativeApproach2018,YANG20181}. There is another kind of safe rules, which is constructed based on feasible solutions, such as GAP \cite{fercoqMindDualityGap2015}. It could guarantee safety more strictly in both theory and real applications. However, it has to be applied repeatedly in the solving process and sometimes that will be inefficient.
	
	Since SSR achieves outstanding performance in $C$-SVM, the natural idea is to further apply this idea to more interpretable and flexible models of type $\nu$-SVM. In fact, it is difficult to figure this out. There are at least two important bottlenecks.
	
	First, in $C$-SVM, the constraint conditions of the dual problem could be rewritten without parameters. In this way, no matter how the parameter $C$ is set, the feasible region of the optimization problem will not change. In this paper, we call this trait the invariance property of the feasible region (IPFR). The IPFR provides a very suitable condition for establishing SSR. However, $\nu$-SVM does not have IPFR, which brings great difficulty in establishing SSR to reduce computational cost. This is the most important problem in our work.
	
	Second, in the original $C$-SVM, the two support hyperplanes are constructed as $\bm{w}\cdot \bm{x}+b = \pm 1$,  only related to the variables $\bm{w}$ and $b$ to be solved. For comparison, in $\nu$-SVM, the two support hyperplanes $\bm{w}\cdot \bm{x}+b = \pm \rho$ are more flexible. When a new variable $\rho$ is added, the feasible region of solutions is more difficult to estimate. A more detailed discussion is given in Section 2.l.
	
	Although \cite{yuanBoundEstimationbasedSafe2021} derived a safe screening rule for the maximum margin of twin spheres support vector machine with pinball loss (SSR-Pin-MMTSM), and applied it to $\nu$-SVM at the end of the paper, its safe region for this method is not tight enough to identify as many redundant samples as possible. Furthermore, it did not provide the corresponding derivations and algorithms in detail. Therefore, the above questions still need some exploration.
	
	In this paper, we overcome the theoretical difficulties above and propose the SSR with bi-level optimization for $\nu$-SVM type models. The main contributions are as follows:
	
	\begin{enumerate}
	\item The bottleneck above is overcome by constructing a bi-level optimization problem, so that the screening proportion and the computational cost could be ideally balanced.

	\item Our proposed method could greatly accelerate the original $\nu$ -SVM with safety.

	\item It is the first time that the idea of safe screening is introduced into an unsupervised problem, i.e., OC-SVM. Our work provides guidance for raising a screening rule for sparse optimization with parameter constraints.
	\end{enumerate}
	
	The rest of this paper is organized as follows. Section \ref{Preliminaries} simply reviews the $\nu$-SVM and then analyzes its Karush-Kuhn-Tucker (KKT) conditions and dual problems. Our SRBO-$\nu$-SVM and DCDM are proposed in Section \ref{SSR-nu-SVM}. Section \ref{GeneralDiscussion} gives a general discussion on the proposed SRBO and further provides SRBO-OC-SVM. Section \ref{experiment} conducts numerical experiments on artificial data sets and benchmark data sets to verify the safety and validity of our proposed method. The last part is the conclusion.
			
	\section{Preliminaries}\label{Preliminaries}
	In this section, the basic knowledge of $\nu$-SVM and the motivation of our proposed SRBO-$\nu$-SVM are given.
	\subsection{The Model of \texorpdfstring{$\nu$}.-SVM}
	Given training vectors $\bm{{x}_{i}}\in\mathbb{R}^{p},i=1,2,\cdots,l$ of two classes, and a label vector $\bm{Y} \in\mathbb{R}^{l}$ such that $y_{i}\in \{1,-1\}$, $\nu$-SVM solves the following primal problem
	\begin{eqnarray}\label{001}
		\displaystyle{\min_{\bm{w},\bm{\xi},\rho}}~~&& \frac{1}{2}\|\bm{w}\|^{2}-\nu\rho+\frac{1}{l}\sum_{i=1}^{l}\xi_{i}\\
		\mbox{s.t.}~~&&y_i(\langle	\bm{w},\Phi(\bm{x_{i}})\rangle+b) \geq \rho -\xi_{i}, \nonumber\\
		~~&&\rho \geq 0,~\xi_{i}\geq 0, i=1,2,\cdots,l, \nonumber
	\end{eqnarray}	
	where $\bm{w}$, $b$ and $\rho$ are the variables of the support hyperplanes $\langle	\bm{w},\Phi(\bm{x})\rangle+b = \pm\rho$. $\xi_{i}$ is the slack factor. $\nu$ is a parameter manually selected that could control the proportions of support vectors and misclassified samples in training.

    As an excellent variant of $C$-SVM, $\nu$-SVM is unique in its own way. It seems that the rewritten hyperplanes $(\bm{w}/\rho) ^T \bm{x}+b/\rho=\pm 1$ are equivalent to $\bm{w}_{new}^T \bm{x} + b_{new} = \pm 1$. However, looking at the optimization problem (1) of $\nu$-SVM, the objective function is quadratic with respect to $\bm{w}$. Therefore, the formulation rewritten by simple linear scaling with $\rho$ is not equivalent to $C$-SVM. In other words, the existing solvers or algorithms of traditional $C$-SVM cannot directly work for $\nu$-SVM.

	For the sake of discussion, let $\Phi(\bm{x_{i}})\leftarrow[\Phi(\bm{x_{i}}),1], \bm{w}\leftarrow[\bm{w};~b]$. Then, the formulation (\ref{001}) can be rewritten as
	 \footnote{Note that, the first term $\frac{1}{2}\|\bm{w}\|^{2}$ in (\ref{001a}) stands for $\frac{1}{2}(\|\bm{w}\|^{2}+b^{2})$ of previous $\nu$-SVM. In this sense, it is not a strict derivation from (\ref{001}) to (\ref{001a}). But the added term $b^{2}$ could make the optimization problem more stable (to guarantee the achievement of the global optimal $b$). Some literatures have well studied this issue \cite{yuan-haishaoImprovementsTwinSupport2011,xuImprovedNtwinSupport2014}, and researchers named the models included $\frac{1}{2}(\|\bm{w}\|^{2}+b^{2})$ as \emph{bounded} SVMs.}  
	\begin{eqnarray}\label{001a}
		\displaystyle{\min_{w,\xi,\rho}}~~&& \frac{1}{2}\|\bm{w}\|^{2}-\nu\rho+\frac{1}{l}\sum_{i=1}^{l}\xi_{i}\\
		\mbox{s.t.}~~&&y_i \langle	\bm{w},\Phi(\bm{x_{i}})\rangle \geq \rho -\xi_{i}, \nonumber\\
		~~&&\rho \geq 0,~\xi_{i}\geq 0, i=1,2,\cdots,l. \nonumber
	\end{eqnarray}
	After $\bm{w}$ and $\rho$ are obtained, the label of a new test sample $\bm{x_{0}}$ can be predicted by the following decision function
	\begin{eqnarray}\label{decision1}
		g(\bm{x_{0}})=\textmd{sgn}(\langle	\bm{w},\Phi(\bm{x_{0}})\rangle).
	\end{eqnarray}
 
	\begin{figure}[htbp!]
		\centering
		\includegraphics[width=0.5\textwidth]{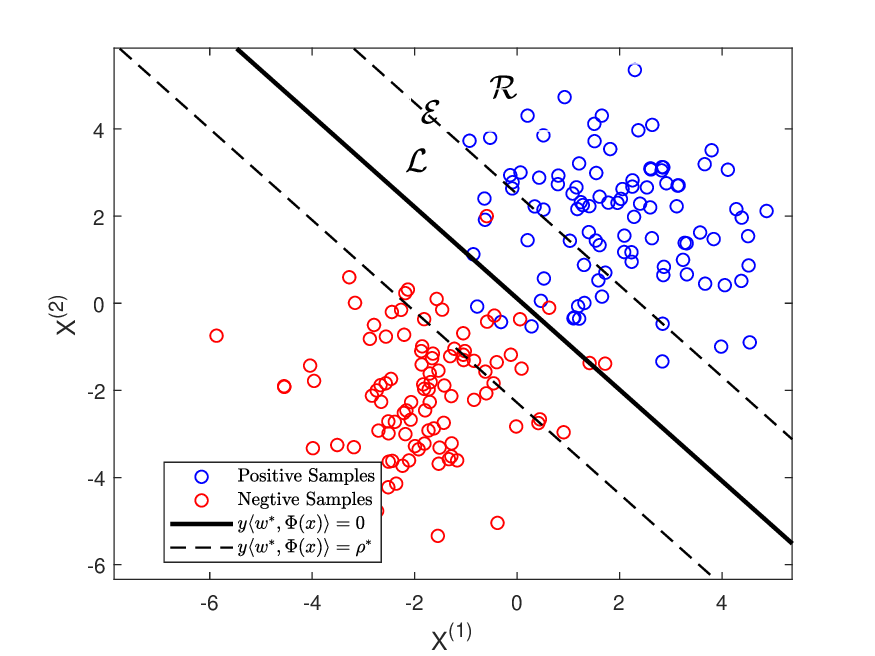}
		\caption{An illustration of $\nu$-SVM on a 2-D artificial data set. $X^{(1)}$ and $X^{(2)}$ represent two features of the samples, respectively. The black dotted lines are support hyperplanes of two classes. The solid black line represents the decision hyperplane.}
		\label{fig0}
	\end{figure}
	
	An illustration of $\nu$ -SVM on a 2D artificial data set is shown in Fig. \ref{fig0}. Two black dotted lines represent positive and negative hyperplanes (support hyperplanes), respectively. First, they separate the two classes of samples as correctly as possible, that is, $y_i \langle \bm{w}\cdot \Phi(\bm{x_{i}}) \rangle \geq \rho  -\xi_{i}$ and minimize $\frac{1}{l}\sum_{i=1}^{l}\xi_{i}$. Second, it is required that the margin between these two hyperplanes be as large as possible, that is, $\max \frac{2\rho}{\|\bm{w}\|}$. In fact, the parameter $\nu$ is a tunable trade-off between the two above points.
	
	For convenience in using the kernel technique and converting the optimization problem (\ref{001a}) to a standard form, we commonly derive its dual problem by introducing the Lagrange function and the KKT conditions. 
	
	Finally, the dual problem of $\nu$-SVM can be rewritten as
	\begin{eqnarray}\label{dualQPPv}
		\min_{\bm{\alpha} \in \mathcal{A}_{\nu}}~~F(\bm{\alpha}),
	\end{eqnarray}
	where $F(\bm{\alpha}) = \frac{1}{2}\bm{\alpha}^{T}\bm{Q}\bm{\alpha}$, $\bm{Q}=\mathrm{diag}(\bm{Y})\kappa(\bm{X},\bm{X})\mathrm{diag}(\bm{Y})$, and $\mathcal{A}_{\nu} = \{\bm{\alpha} ~|~ \bm{e^{T}}\bm{\alpha}\geq \nu, 0\leq\bm{\alpha}\leq\frac{1}{l}\}$.
	
    Specifically, the KKT conditions consists of the following formulations.
	\begin{equation}\label{KKT}
		\begin{aligned}
			&\bm{w^*}-\sum_{i=1}^{l}\alpha_{i}^*y_{i}\Phi(\bm{x_{i}})^T=0,~\sum_{i=1}^{l}\alpha_{i}^{*}-\nu-\gamma^*=0,~\frac{1}{l}-\alpha_{i}^*-\beta_{i}^*=0,~\gamma^*\geq 0,~\gamma_{i}^*\rho^*=0,\\
			&\alpha_{i}^*(y_{i}\bm{w^*}\cdot \Phi(\bm{x_{i}})-\rho^*+\xi_{i}^*)=0,~\beta_{i}^*\xi_{i}^*=0,~\bm{\alpha^*}\geq 0,~\beta^*\geq 0,~y_{i}\bm{w^*} \cdot \Phi(\bm{x_i})\geq \rho^*-\xi^*_{i},
		\end{aligned}
	\end{equation}
	where $\bm{\alpha^*}, \beta^*,\gamma^*, \bm{w^*}, \rho^* and~\xi_{i}^*$ are the optimal solutions.

	The optimal solution $\bm{w^*}$ of (\ref{001a}) can be obtained by $\bm{w^*}={\Phi{(\bm{X})}}^T\mathrm{diag}(\bm{Y})\bm{\alpha^*}$.
	
	When the optimal solution $\bm{\alpha^*}$ of the problem (\ref{dualQPPv}) is obtained, the decision function (\ref{decision1}) can be calculated by
	\begin{eqnarray}\label{decision2}
		g(\bm{x_{0}})=\textmd{sgn}(\kappa(\bm{x_{0}},\bm{X}) \mathrm{diag}(\bm{Y})\bm{\alpha^*}).
	\end{eqnarray}

    For notational convenience, the main symbols used in this paper are summarized in Table \ref{symbols and definitions}.
 \begin{table}[htbp!]
		\centering
		\caption{Description of Symbols  }\label{symbols and definitions}
		\resizebox{\linewidth}{!}
  {
		\begin{tabular}{cc}\hline
			Symbols & Description  \\
			\noalign{\smallskip} \hline \noalign{\smallskip}
		    $\bm{x_{i}}$
      &The $i^{\mathrm{th}}$ samples with $p$ attributes\\
		    $\bm{X} = [\bm{x_1}, \bm{x_2}, \cdots, \bm{x_l}]^{T}$
      & The matrix of training set that contains $l$ samples\\
		    $\langle \bm{w},\bm{c}\rangle=\sum_{i=1}w_{i}c_{i}$
      &Inner product of $\bm{w},\bm{c}$\\
		    $\|\bm{w}\|^{2}=\langle \bm{w},\bm{w}\rangle$
      &The $l_2$ norm of $\bm{w}$\\
		    $\Phi(\cdot)$
      &A nonlinear mapping that transforms samples to a high-dimensional feature space\\
		    $\kappa(\bm{x_i},\bm{x_j})=\langle \Phi(\bm{x_i}),\Phi(\bm{x_j})\rangle$
      &Kernel function in Hilbert space\\
		    $\bm{\alpha}$
      &Feasible solution in an optimization problem\\
		    $\bm{\alpha^{*}}$
      &Optimal solution corresponding to $\bm{\alpha}$\\
		    $\bm{e}$&An appropriate dimensional ones vector\\
		    $|\mathcal{S}|$
      &The cardinality of a given set $\mathcal{S}$\\
		    \bottomrule
		\end{tabular}}
          \begin{tablenotes}
          \footnotesize
           \item[1] In this paper, all vectors and matrices are represented with boldface letters.
      \end{tablenotes}
	\end{table}
 
	\subsection{The Sparsity of Dual Solution \texorpdfstring{$\bm{\alpha^*}$}.}    
	Here, three sample index sets are defined as follows.
	\begin{equation}\label{DefineERL}
		\begin{aligned}
			&\mathcal{E}=\{i~|~y_i \langle	\bm{w^*},\Phi(\bm{x_{i}})\rangle=\rho^{*}, i=1,2,\cdots,l \},\\
			&\mathcal{R}=\{i~|~y_i \langle	\bm{w^*},\Phi(\bm{x_{i}})\rangle>\rho^{*}, i=1,2,\cdots,l \},\\
			&\mathcal{L}=\{i~|~y_i \langle	\bm{w^*},\Phi(\bm{x_{i}})\rangle<\rho^{*}, i=1,2,\cdots,l \}.
		\end{aligned}
	\end{equation}
	
	The index sets $\mathcal{R}$, $\mathcal{E}$, and $\mathcal{L}$ corresponding to positive samples are shown in Fig. \ref{fig0}. The positive hyperplane divides the positive samples into three parts. The samples in $\mathcal{E}$ correspond to them exactly lying in the support hyperplane $y_i \langle	\bm{w^*},\Phi(\bm{x_{i}})\rangle=\rho^{*}$. The samples in $\mathcal{R}$ correspond to the index of these properly separated by the support hyperplane. The samples in $\mathcal{L}$ correspond to the index of these misclassified by the support hyperplane. Similar results can be obtained for negative samples.
	
	By analyzing the KKT conditions (\ref{KKT}), the following results can be easily drawn
	\begin{eqnarray}\label{L022}	
		i \in \mathcal{E} \Longrightarrow 0\leq \alpha_{i}^{*} \leq \frac{1}{l} ,
	\end{eqnarray}   
	\begin{eqnarray} \label{L021}
		i \in \mathcal{R} \Longrightarrow \alpha_{i}^{*}=0,
	\end{eqnarray}
	\begin{eqnarray}\label{L023}
		i \in \mathcal{L} \Longrightarrow \alpha_{i}^{*}=\frac{1}{l} .
	\end{eqnarray}
	
	The length of the solution vector $\bm{\alpha}$ of the dual problem (\ref{dualQPPv}) is consistent with the number of training samples. The above results further imply that we could determine the value of $\bm{\alpha^*}$ by observing the situations of samples separated by supported hyperplanes. Generally, there are a few samples exactly on the support hyperplane (corresponding to the situation (\ref{L022})). That is, there are only a small number of elements in $\bm{\alpha^*}$ in the range of ($0$, $1/l$). On the contrary, most of the samples are in situations (\ref{L021}) or (\ref{L023}), that is, most of the elements in $\bm{\alpha^*}$ value the constant $0$ or $1/l$. This is what we usually call the sparsity of SVMs. 

 Note that if $\alpha^{*}_{i}=0$, then it satisfies formula (8) or formula (9). Thus, we cannot identify whether the corresponding $i$ is in $\mathcal{E}$. A similar situation occurs when $\alpha^{*}_{i}=1/l$. Only if $0<\alpha^{*}_{i}<1/l$, can we determine the corresponding $i$ is in $\mathcal{E}$.
 
For convenience, we define the samples corresponding to $i \in \mathcal{E}$ as active samples and those instances corresponding to $i \in \mathcal{R}$ or $i \in \mathcal{L}$ as inactive samples.
	
	\subsection{Motivation}
	For the $\nu$-SVM, the training process corresponds to solving the dual problem (\ref{dualQPPv}), which is in the form of a standard QPP with computational complexity $O(l^3)$. The test process is to calculate the decision function (\ref{decision2}) that requires only simple multiplication of the matrix, and its calculation cost is negligible. Therefore, the computational cost of the entire model is determined mainly by the scale of the dual problem (\ref{dualQPPv}), that is, the sample size.
	
	In addition, a parameter $\nu$ is manually selected in the model. To determine its optimal value, the common approach is cross-validation with grid search. Therefore, for each value of $\nu$ taken, the dual problem (\ref{dualQPPv}) involving all samples must be solved repeatedly. This is a daunting computational challenge for large-scale data sets. 
	
	Fortunately, considering the solution sparsity indicated in (\ref{L021}) and (\ref{L023}), if we can precisely identify the inactive samples as much as possible, the corresponding elements of $\bm{\alpha^*}$ that are constant can be determined. Thus, it only needs to solve a smaller optimization problem, and it guarantees to achieve the same solution as the original model. It will greatly reduce computational overhead by embedding this approach in the parameter selection of the grid search. Furthermore, when solving QPP, a traditional approach is to use MATLAB's own quadratic programming solution function `quadprog', but we would like to propose a fast solution method for $\nu$-SVM to improve efficiency. This is the motivation of our work in this paper.
	
	\section{A Safe Screening Rule with Bi-level Optimization for \texorpdfstring{$\nu$}.-SVM} \label{SSR-nu-SVM}

	\subsection{The Basic Rule} \label{SecBasicRule}
	Based on the preliminaries, our most important issue is to detect samples that belong to $\mathcal{R}$ and $\mathcal{L}$. As can be seen in (\ref{DefineERL}), when given training data, the sets $\mathcal{R}$ and $\mathcal{L}$ depend only on $\bm{w^*}$ and $\rho^*$ that must be obtained by solving the primal problem (\ref{001a}). That seems like a paradox.
	
	Instead of solving the optimization problem (\ref{001a}), we managed to give a feasible region $\mathcal{W}$ that includes $\bm{w^*}$, that is, $\bm{w^*} \in \mathcal{W}$, and provide a feasible lower bound $\rho_{\mathtt{lower}}$ and upper bound $\rho_{\mathtt{upper}}$ of $\rho^*$, i.e. $\rho_{\mathtt{lower}}\leq \rho^{*} \leq  \rho_{\mathtt{upper}}$. Then, (\ref{L021}) and (\ref{L023}) can be transformed into 
	\begin{equation}\label{L026} 
		\begin{aligned}
			&\inf \limits_{\bm{w}\in \mathcal{W}} \{y_i \langle	\bm{w^*},\Phi(\bm{x_{i}})\rangle \}>\rho_{\mathtt{upper}}
			\Longrightarrow  \alpha_{i}^{*}=0, \\
			&\sup \limits_{\bm{w} \in \mathcal{W}} \{y_i \langle	\bm{w^*},\Phi(\bm{x_{i}})\rangle \}<\rho_{\mathtt{lower}}
			\Longrightarrow   \alpha_{i}^{*}=\frac{1}{l}.
		\end{aligned}
	\end{equation}
	This is the basic rule of our screening method.
	
	\subsection{Feasible Region \texorpdfstring{$\mathcal{W}$}.}\label{SubsectionFeasibleRegionW} 
	To find the feasible region $\mathcal{W}$ included $\bm{w^*}$, firstly, we introduce the following lemma.   
	\newtheorem{lemma}{Lemma}
	\begin{lemma} \label{lemmavi}
		(Variational Inequality) \cite{gulerFoundationsOptimization2010} Let $F(\bm{\alpha})$ be a differentiable function on the open set containing the convex set $\mathcal{A}$. When $\bm{\alpha^{*}}$ is the local minimum of $F(\bm{\alpha})$, the following inequality holds
		\begin{eqnarray}\label{L027}
			\langle\nabla F(\bm{\alpha^{*}}),\bm{\alpha}-\bm{\alpha^{*}}\rangle \geq 0.
		\end{eqnarray}
		Here, $\nabla F(\bm{\alpha^{*}})$ denotes the gradient of function $F$ at $\bm{\alpha^{*}}$.
	\end{lemma}
	
	In addition, without loss of generality, assume that given a parameter value $\nu_0$, the corresponding optimal solution $\bm{\alpha^*}$ for the dual problem (\ref{dualQPPv}) is achieved. The task is to find a feasible region $\mathcal{W}$ included $\bm{w^*}$ under a new parameter value $\nu_1$, where $0 < \nu_{0} < \nu_{1} <1$.
	
	For convenience of notation, the optimal solutions of (\ref{dualQPPv}) under the parameter values $\nu_0$ and $\nu_1$ are denoted by $\bm{\alpha^{0}}$ and $\bm{\alpha^{1}}$, respectively. That is
	\begin{eqnarray*}
		\bm{\alpha^{0}}= \arg \min_{\bm{\alpha} \in \mathcal{A}_{\nu_0}}  F(\bm{\alpha}),~~ 
		\bm{\alpha^{1}}= \arg \min_{\bm{\alpha} \in \mathcal{A}_{\nu_1}}  F(\bm{\alpha}).	
	\end{eqnarray*}
	Similarly, $\bm{w_{0}}$ and $\bm{w_{1}}$ represent the optimal solution of the primal problem (\ref{001a}) under the parameters $\nu_0$ and $\nu_1$, respectively. Thus, we obtain the following theorem.
	
	\newtheorem{theorem}{Theorem}
	\begin{theorem}\label{theorem01}
		For problem (\ref{001a}), given $0 < \nu_{0}<\nu_{1} <1$, the following holds 
		\begin{eqnarray}\label{L029}
			\|\bm{w_{1}}-\bm{c}\|^{2}\leq r.
		\end{eqnarray}
		Here, $\bm{c}=\bm{w_{0}}+\frac{1}{2}\bm{Z^{T}}\bm{\delta}$ where $\bm{Z}=\mathrm{diag}(\bm{Y})\Phi(\bm{X})$, $\forall \bm{\delta} \in \Delta$, and $\Delta = \{\bm{\delta} | \bm{e^{T}}(\bm{\alpha^{0}}+\bm{\delta})\geq \nu_1, 0\leq\bm{\alpha^{0}} +\bm{\delta} \leq \frac{1}{l}\}$. $r=\bm{c^{T}}\bm{c}-\bm{w_{0}^{T}}\bm{w_{0}}$.
	\end{theorem}
	
	\begin{proof}   
		According to lemma \ref{lemmavi}, we have
		\begin{eqnarray}\label{028}
				\langle\nabla F(\bm{\alpha^{1}}),\bm{\alpha^{0}}+\bm{\delta}-\bm{\alpha^{1}}\rangle \geq 0, ~~~~
    \langle\nabla F(\bm{\alpha^{0}}),\bm{\alpha^{1}}-\bm{\alpha^{0}}\rangle \geq 0.
		\end{eqnarray}
		Where $\bm{\delta}$ can be any vector satisfying $\bm{\alpha^{0}}+\bm{\delta} \in \mathcal{A}_{\nu_1}$, namely    
		\begin{eqnarray}\label{delta1}
				\bm{e^{T}}(\bm{\alpha^{0}}+\bm{\delta})\geq \nu_1, ~~~~0\leq\bm{\alpha^{0}} +\bm{\delta} \leq \frac{1}{l}.
		\end{eqnarray}
		
		Combining the inequality (\ref{028}) with $\nabla F(\bm{\alpha})=\bm{Q}\bm{\alpha}$, we can get the folowing inequation		
		\begin{eqnarray*}
			(\bm{\alpha^{1}})^{T}\bm{Q}\bm{\alpha^{1}}-(\bm{\alpha^{1}})^{T}\bm{Q}(2\bm{\alpha^{0}}+\bm{\delta}) + \frac{1}{4}(2\bm{\alpha^{0}}+\bm{\delta})^{T}\bm{Q}(2\bm{\alpha^{0}}+\bm{\delta}) \\
			\leq -(\bm{\alpha^{0}})^{T}\bm{Q}\bm{\alpha^{0}} + \frac{1}{4}(2\bm{\alpha^{0}}+\bm{\delta})^{T}\bm{Q}(2\bm{\alpha^{0}}+\bm{\delta}).
		\end{eqnarray*}
		
		Let $\bm{Z}=\mathrm{diag}(\bm{Y})\Phi(\bm{X})$. We have $\bm{Q} = \bm{Z}\bm{Z^{T}}$. Denoting $\bm{c} = \frac{1}{2}\bm{Z^{T}}(2\bm{\alpha^{0}}+\bm{\delta})$, we can obtain
		$$
		\|\bm{Z^{T}}\bm{\alpha}^1 - \bm{c}\|^2  \leq \bm{c^{T}}\bm{c}-(\bm{\alpha^{0}})^{T}\bm{Q}\bm{\alpha^{0}}.
		$$
		
		According to $\bm{w_{1}}=\bm{Z^{T}}\bm{\alpha}^1$ and $\bm{w_{0}^{T}}\bm{w_{0}}=(\bm{\alpha^{0})^{T}}\bm{Q}\bm{\alpha^{0}}$, $c$ can also be written as $\bm{w_{0}}+\frac{1}{2}\bm{Z^{T}}\bm{\delta}$.
		
		This completes the proof.
	\end{proof} 
	
	Based on Theorem \ref{theorem01}, we can define a spherical region $\mathcal{W} = \{\bm{w}~|~\|\bm{w}-\bm{c}\|^{2}\leq r\}$ that must contain $\bm{w_{1}}$. The geometric implication of Theorem \ref{theorem01} can be shown in Fig. \ref{Region}. $\bm{w_{1}}$ is located in a sphere $\mathcal{W}$ with center $\bm{c}$ and radius $r ^{\frac{1}{2}}$. 
	
	\newtheorem{corollary}{Corollary}
	\begin{corollary}\label{corollary1} For optimization problem (\ref{001a}), given parameter values $\nu_{0}$ and $\nu_{1}$, the following holds 
		\begin{eqnarray}
			\inf \limits_{\bm{w}\in \mathcal{W}} \{y_i \langle	\bm{w},\Phi(\bm{x_{i}})\rangle \} \geq \bm{Z_{i}}\bm{c} - \|\bm{Z_{i}}\|\cdot| r|^{\frac{1}{2}}\label{InfBound},\\
			\sup \limits_{\bm{w}\in \mathcal{W}} \{y_i \langle	\bm{w},\Phi(\bm{x_{i}})\rangle \} \leq \bm{Z_{i}}\bm{c} + \|\bm{Z_{i}}\|\cdot| r|^{\frac{1}{2}}\label{SupBound},
		\end{eqnarray}
		where $\bm{Z}=\mathrm{diag}(\bm{Y})\Phi(\bm{X})$, $\mathcal{W} = \{\bm{w}~|~\|\bm{w}-\bm{c}\|^{2}\leq r\}$, $\bm{c}=\bm{w_{0}}+\frac{1}{2}\bm{Z^{T}}\bm{\delta}$, $\forall \bm{\delta} \in \Delta$, and $r=\bm{c^{T}}\bm{c}-\bm{w_{0}^{T}}\bm{w_{0}}$.
	\end{corollary}
 
	\begin{proof} 
 For all $\bm{w} \in \mathcal{W}$, we have
		\begin{equation*}
				y_i \langle \bm{w},\Phi(\bm{x_{i}})\rangle 
				= \bm{Z_{i}}(\bm{w} - \bm{c}) + \bm{Z_{i}}c\\	
				\geq \bm{Z_{i}}\bm{c} -\|\bm{Z_{i}}\|\cdot \|\bm{w}-\bm{c}\| \geq \bm{Z_{i}}\bm{c} -\|\bm{Z_{i}}\|\cdot| r|^{\frac{1}{2}}.
		\end{equation*}
		The first inequality comes from the Cauchy-Schwarz inequality, and the second inequality is based on the definition of $\mathcal{W}$. Then, the inequality (\ref{InfBound}) can be easily proved.
		
		Similarly, the inequality (\ref{SupBound}) can be proved.
	\end{proof}

	(\ref{SupBound}) and (\ref{InfBound}) give the upper and lower bounds on the left side of the inequalities (\ref{L026}) that constitute the basic rule of our screening method discussed in Section 3.1.

	\begin{figure}[htbp!]
		\centering
		\subfloat[]{\label{Region1}
			\includegraphics[width=0.35\textwidth]{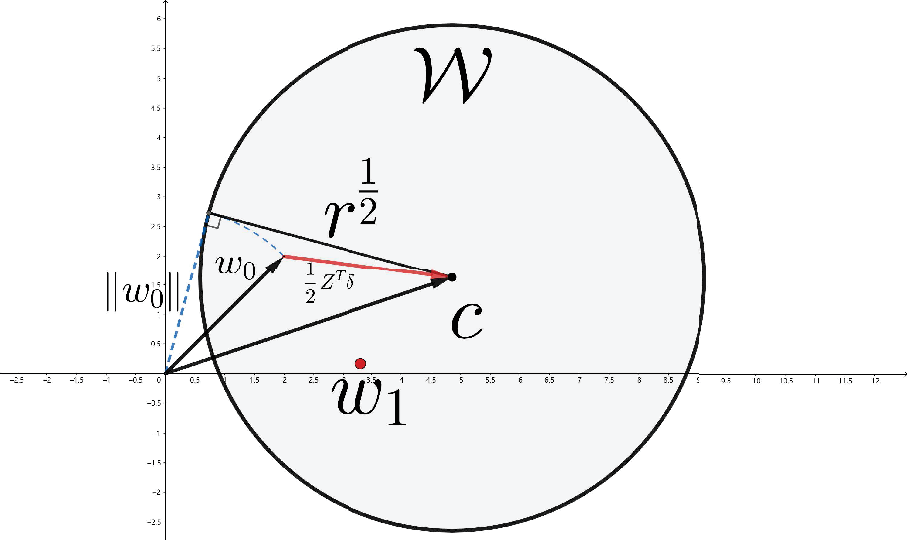}}
		\subfloat[]{\label{Region2}
			\includegraphics[width=0.35\textwidth]{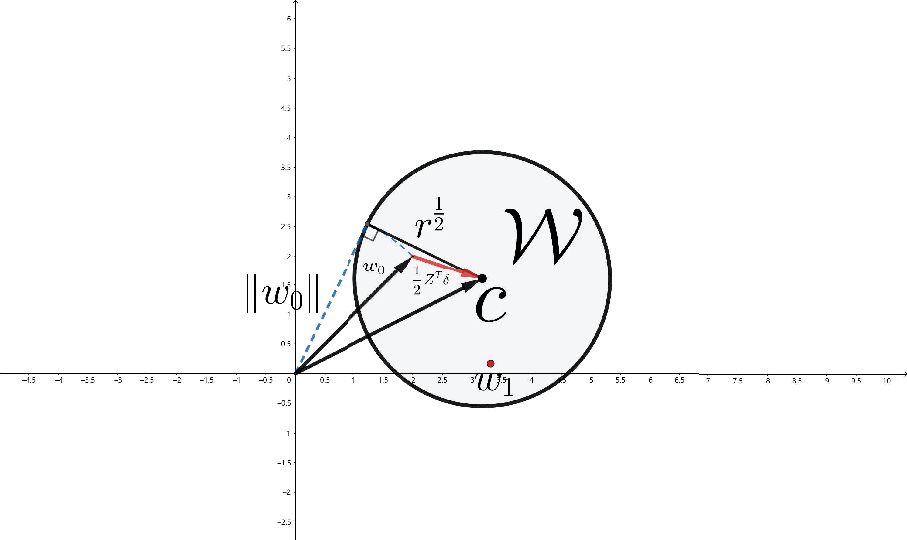}}
		\caption{Illustrations of $\mathcal{W}$ when given $\bm{w_{0}}$. $\mathcal{W}$ is a spherical region with center $c$ and radius $r ^{\frac{1}{2}}$. The difference between (a) and (b) is in the size of the sphere, which is only due to the different selection of $\bm{\delta}$.}
		\label{Region}
	\end{figure}
 
	In addition, as illustrated in Fig. \ref{Region}, the size of the sphere varies with the selection of $\bm{\delta}$. Specifically, when the training data and $\bm{\alpha^{0}}$ (or $\bm{w_{0}}$) are given, the center $\bm{c}$ is fixed and the radius $r ^{\frac{1}{2}}$ will take different values depending on the choice of $\bm{\delta}$. To estimate $\bm{w_{1}}$ as accurately as possible, the smaller $r$ is, the better. Regarding $r$ as a function of $\bm{\delta}$, that is, $r(\bm{\delta}) = \frac{1}{4}\bm{\delta}^{T}\bm{Q}\bm{\delta} + \bm{\alpha^{0}}\bm{Q}\bm{\delta}$, the optimal $\bm{\delta}$ can be obtained by
	\begin{equation}\label{deltaQPP}
		\begin{aligned}
			\bm{\delta}^*&=\arg \min_{\bm{\delta}\in \Delta}~r(\bm{\delta}),
		\end{aligned}
	\end{equation}
	where $\Delta = \{\bm{\delta} | \bm{e^{T}}(\bm{\alpha^{0}}+\bm{\delta})\geq \nu_1, 0\leq\bm{\alpha^{0}} +\bm{\delta} \leq \frac{1}{l}\}$. The formulation above also involves a QPP, and its size is determined by the scale of training samples. 
	
	Although taking $\bm{\delta}^*$ guarantees the best performance to estimate $\bm{w_{1}}$, it brings additional computational overhead to solve the QPP (\ref{deltaQPP}). To compensate for the screening proportion and computational cost, we establish a bi-level optimization structure by introducing a variable $\bm{\delta}$, which has not been mentioned in related previous papers. An efficient algorithm to calculate $\bm{\delta}$ is described in Section 3.5.
		
	\subsection{Upper and Lower Bounds of \texorpdfstring{$\rho^*$}.}
	To estimate the upper bound $\rho_{\mathtt{upper}}$ and the lower bound $\rho_{\mathtt{lower}}$ in inequalities (\ref{L026}) for our basic rule, the $\nu$-property in $\nu$-SVM is introduced.
	
	\begin{lemma}\label{lemmaNu} ($\nu$-property) \cite{scholkopfNewSupportVector2000} For $\nu$-SVM, support vector set $\mathcal{S}$ and margin error sample set $\mathcal{M}$ are define as
		\begin{eqnarray*}		
				\mathcal{S}=\{\bm{x_{i}}~|~\alpha_{i}^{*}\neq 0, i=1,2,\cdots l\},~~~~
				\mathcal{M}=\{\bm{x_{i}}~|~y_i \langle	\bm{w^*},\Phi(\bm{x_{i}})\rangle<\rho^{*}, i=1,2,\cdots l\}, 
		\end{eqnarray*}    	    	
		$s = |\mathcal{S}|$ and $m = |\mathcal{M}|$. Then, the following holds 
		\begin{eqnarray}
			\frac{m}{l}\leq \nu \leq \frac{s}{l}.
		\end{eqnarray}
	\end{lemma}
	
	The lemma \ref{lemmaNu} implies that $\nu$ is an upper bound on the fraction of margin errors and is a lower bound on the fraction of support vectors.
	
	Based on the $\nu$-property, we can get the following theorem to estimate $\rho_{\mathtt{upper}}$ and $\rho_{\mathtt{lower}}$.
	\begin{theorem}\label{rhoBound}
		Define $d(i')=y_{i'} \langle \bm{w^*},\Phi(\bm{x_{i'}})\rangle$ where $i'$ the index of training samples, which is in descending order, i.e. $d(1) > d(2)> \cdots > d(l)$. Then, the following holds
		\begin{equation}
			d(\lceil i^* \rceil)\leq\rho^{*}\leq d(\lfloor i^* \rfloor),	
		\end{equation}
		where $i^* = l-\nu l$ and $\lceil i^* \rceil$ and $\lfloor i^* \rfloor$ denote the ceil and floor operations of $i^*$, respectively.
	\end{theorem}
	\begin{proof}
		For any samples $\bm{x_{i'}} \notin \mathcal{S}$, we have $i' \in \mathcal{R}$. And because $s\geq \nu l$ from Lemma \ref{lemmaNu}, we get $|\mathcal{R}|<l-\nu l$. According to the definition of $d(i')$, we have $\forall i' \notin \mathcal{R}$, $d(i') \leq \rho^*$. Thus, we achieve $d(\lceil l-\nu l \rceil)\leq\rho^{*}$.
		
		For any sample $\bm{x_{i'}} \notin \mathcal{M}$, we have $i' \in \mathcal{E} \cup \mathcal{L}$. And because $m \leq \nu l$ from Lemma \ref{lemmaNu}, we get $|\mathcal{E} \cup \mathcal{L}|>l-\nu l$. According to the definition of $d(i')$, we have $\forall i' \in \mathcal{E} \cup \mathcal{L} \mathcal{R}$, $d(i') \geq \rho^*$. Thus, we achieve $\rho^{*}\leq d(\lfloor l-\nu l \rfloor)$.
		
		This completes the proof.
	\end{proof}
	
	Combining Corollary \ref{corollary1} and Theorem \ref{rhoBound}, we can easily find the upper and lower bounds of $\rho^*$ by the following Corollary \ref{corollary2}.
	
	\begin{corollary}\label{corollary2} For optimization problem (\ref{001a}), given parameter values $\nu_{0}$ and $\nu_{1}$, $\rho_{\mathtt{upper}}$ and $\rho_{\mathtt{lower}}$ are defined as 
		\begin{equation}\label{rhoUL}
				\rho_{\mathtt{upper}}=\bm{Z_{\lfloor i^*\rfloor}}\bm{c}+|r|^{\frac{1}{2}}\|\bm{Z_{\lfloor i^*\rfloor}}\|, 
			~~~~\rho_{\mathtt{lower}}=\bm{Z_{\lfloor i^*\rfloor}}\bm{c}-| r|^{\frac{1}{2}}\|\bm{Z_{\lfloor i^*\rfloor}}\|,
		\end{equation}
		where $\bm{Z}=\mathrm{diag}(\bm{Y})\Phi(\bm{X})$, $\bm{c}=\bm{w_{0}}+\frac{1}{2}\bm{Z^{T}}\bm{\delta}$, $\forall \bm{\delta} \in \Delta$, $r=\bm{c^{T}}\bm{c}-\bm{w_{0}^{T}}\bm{w_{0}}$, $i^* = l-\nu l$. The index of samples is sorted in descending order by the values of $y_{i} \langle	\bm{w^*},\Phi(\bm{x_{i}})\rangle$. Then, the following holds
		\begin{equation*}
			\rho_{\mathtt{lower}}\leq \rho^{*} \leq  \rho_{\mathtt{upper}}.
		\end{equation*}
	\end{corollary}	

 Note that Eq. (\ref{rhoUL}) implies that the distance between $\rho_{\mathtt{upper}}$ and $\rho_{\mathtt{lower}}$ is $2|r|^{\frac{1}{2}} \|\bm{Z_{\lfloor i^*\rfloor}}\| $.
 
	\subsection{The Proposed Safe Screening Rule with Bi-level Optimization for \texorpdfstring{$\nu$}.-SVM}
	Taking Corollary \ref{corollary1} and Corollary \ref{corollary2} into the basic rule of (\ref{L026}), the proposed SRBO for $\nu$-SVM can be derived.
	
	\begin{corollary}\label{corollary3} (SRBO-$\nu$-SVM)
		For problem (\ref{001a}), suppose that the parameter value $\nu_{0} > 0$ and the corresponding optimal solution $\bm{\alpha^{0}}$ are given. Then, for a parameter value $\nu_{1}$ ( $\nu_{1} > \nu_{0} > 0$) and the corresponding optimal solution $\bm{\alpha^1}$.
		
		1. We have $\alpha^{1}_{i}=0$, if the following holds
		\begin{eqnarray}\label{v_alpha0}
			\bm{Z_{i}}\bm{c}-| r|^{\frac{1}{2}}\|\bm{Z_{i}}\| >\rho_{\mathtt{upper}}.
		\end{eqnarray}
		
		2. We have $\alpha^{1}_{j}=\frac{1}{l}$, if the following holds
		\begin{eqnarray}\label{v_alpha1}
			\bm{Z_{j}}\bm{c}+| r|^{\frac{1}{2}}\|\bm{Z_{j}}\| <\rho_{\mathtt{lower}}.
		\end{eqnarray}
	\end{corollary}

        \begin{figure*}[htbp!]
		\centerline{\includegraphics[width=1.5\textwidth]{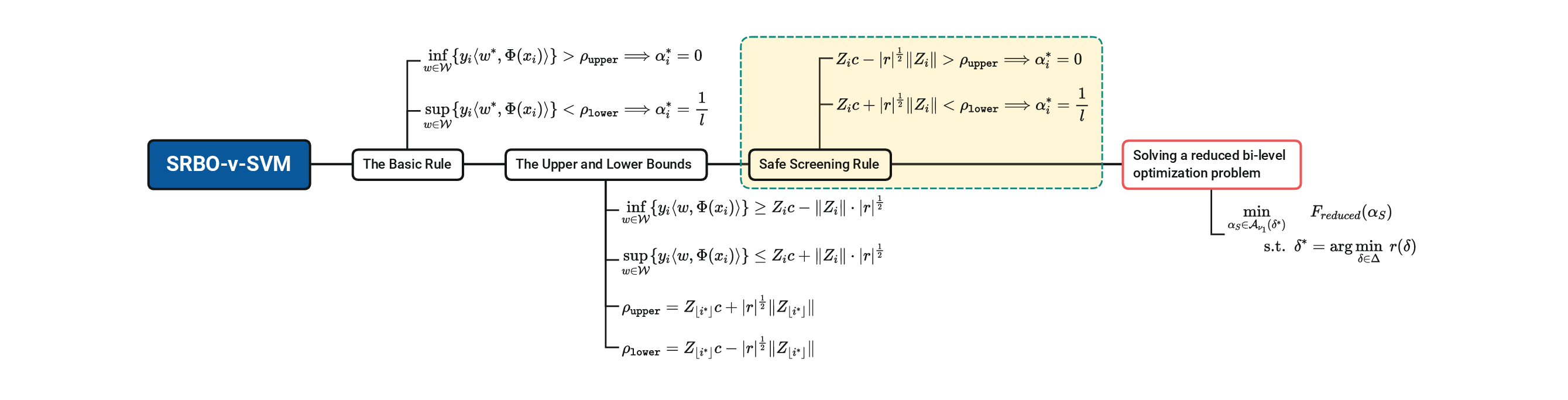}}
		\caption{Mathematical framework of our SRBO-$\nu$-SVM.}
		\label{fig3}
	\end{figure*}
 
	For a clear understanding of our entire derivation of SRBO-$\nu$-SVM, Fig. \ref{fig3} is given. As shown in the figure, the entire derivation is summarized into four points. First, the basic rule (\ref{L026}) is obtained from the KKT condition of the primal problem (\ref{001a}). Second, using the variational inequality and $\nu$-property, the upper and lower bounds related to the basic rule are derived from Corollary \ref{corollary1} and Corollary \ref{corollary2}. Third, we obtain the safe screening rule of $\nu$-SVM. Finally, to obtain the solution with parameter $\nu_1$ when given $\nu_0$ and the corresponding solution $\bm{\alpha^0}$, only a small-scale optimization problem is required to solve.  
	\begin{eqnarray*}\label{L030}
		\min_{\bm{\alpha_S} \in \mathcal{\bar{A}}_{\nu_{1}}}&F_{reduced}(\bm{\alpha_S}).
	\end{eqnarray*}
	Here, $F_{reduced}(\bm{\alpha_S}) = \frac{1}{2}\bm{\alpha_{S}^{T}}\bm{Q_{1}}\bm{\alpha_{S}}+\bm{f^{T}}\bm{\alpha_{S}}$, $\mathcal{\bar{A}}_{\nu_{1}} = \{\bm{\alpha_S}~|\bm{e^{T}}\bm{\alpha_{S}}\geq\nu_{1}-\bm{e^{T}}\bm{\alpha_{D}},0\leq{\bm{\alpha_{S}}}\leq \frac{1}{l}\}$,    
	where $D$ denotes the index of the identified samples in $\mathcal{R}$ and $\mathcal{L}$ by the screening rule, i.e. inactive sample index. $S$ denotes the index of the remaining samples, that is, the active sample index, and $\bm{Q_{1}}=\bm{Q_{S,S}},~\bm{f}=\bm{Q_{S,D}}\bm{\alpha_{D}}$.
	
	Note that in order to make the screening rule as efficient as possible, the choice of $\bm{\delta}$ is critical. The optimal $\bm{\delta}^*$ given in (\ref{deltaQPP}) should be considered. 
			
	We further consider embedding the SRBO-$\nu$-SVM in the process of selecting the optimal value of the parameter $\nu$ in the grid search and give the sequential version of the SRBO-$\nu$-SVM.
	
	\begin{corollary} \label{cor-2}
		(Sequential SRBO-$\nu$-SVM) Given the parameter sequence $\nu_{0}<\nu_{1}<\cdots<\nu_{K}$, for any integer $0\leq k \leq K$. Suppose that the primal optimal solution $\bm{w_{k}}$ and the dual optimal solution $\bm{\alpha^{k}}$ under the parameter value $\nu_k$ are known. 
		Let $\bm{c_{k}}=\bm{w_{k}}+\frac{1}{2}\bm{Z^{T}}\bm{\delta}_{k},  r_{k}=\bm{c_{k}^{T}}\bm{c_{k}}-\bm{w_{k}^{T}}\bm{w_{k}}$. 
		$\bm{\delta}_{k}$ can be any vector satisfying $\bm{\alpha^{k}}+\bm{\delta}_{k} \in \mathcal{A}_{\nu_{k+1}}$.	 
		
		1. We have $\alpha^{k+1}_{i}=0$, if the following holds
		\begin{eqnarray}\label{cor2a}
			\bm{Z_{i}}\bm{c_{k}}-| r_{k}|^{\frac{1}{2}}\|\bm{Z_{i}}\|>\rho_{\mathtt{upper}}.
		\end{eqnarray}
		
		2. We have $\alpha^{k+1}_{j}=\frac{1}{l}$, if the following holds
		\begin{eqnarray}\label{cor2b}
			\bm{Z_{i}}\bm{c_{k}}+| r_{k}|^{\frac{1}{2}}\|\bm{Z_{i}}\|<\rho_{\mathtt{lower}}.
		\end{eqnarray} 
		where the definition of $\rho_{\mathtt{upper}}$ and $\rho_{\mathtt{lower}}$ is given in (\ref{rhoUL}).
	\end{corollary}
	
	In summary, the key point of our proposed method is to find the constant elements of $\bm{\alpha^{k+1}}$ before solving the optimization problem, to reduce the total computational cost of the training process.

	\subsection{The Algorithm for SRBO-\texorpdfstring{$\nu$}.-SVM}\label{SectionAlgorithm}
	According to the discussion above, when given the parameter sequence $\{\nu_{k}|k=0,1,\cdots,K\}$, the training procedure of our SRBO-$\nu$-SVM can be summarized as follows.
	
	Step 1 (Initialization). Solving the entire dual problem (\ref{dualQPPv}) under the parameter value $\nu_0$ and getting the corresponding optimal solution $\bm{\alpha^0}$.
	
	Then for each $k =0, 1, \cdots, K-1$, perform the following steps. 
	
	Step 2 (Screening). find an appropriate vector $\bm{\delta}_k$, identify the inactive training samples in $\mathcal{R}$ and $\mathcal{L}$, and directly obtain the corresponding $\bm{\alpha^{k+1}_{D}}$ by the sequential SRBO-$\nu$-SVM given in Corollary \ref{cor-2}.
	
	Step 3 (Optimization). Solve a small-scaled problem and find the corresponding solution
	\begin{equation}\label{QPPforAlaphaD}
		\bm{\alpha^{k+1}_{S}} = \arg \min_{\bm{\alpha_S} \in \mathcal{\bar{A}}_{\nu_{k}}} F_{reduced}(\bm{\alpha_S}).
	\end{equation}
	
	Step 4 (Combination). achieve the entire dual optimal solution $\bm{\alpha^{k+1}}$ by combining $\bm{\alpha^{k+1}_{D}}$ and $\bm{\alpha^{k+1}_{S}}$.
	
	Another consideration is to determine the appropriate $\bm{\delta}_k$. From Fig. \ref{Region}, we can see that the choice of $\bm{\delta}_k$ will severely affect the estimation of $\mathcal{W}$. Too large a feasible range $\mathcal{W}$ may even result in no sample being screened, i.e. $\mathcal{D} = \emptyset$. The QPP (\ref{deltaQPP}) gives the optimal $\bm{\delta}_k$ in theory. However, it should be noted that this optimization problem is also a QPP with $l$ variables. It will increase the computational overhead of our SRBO to directly solve this problem. That is, we have to make a trade-off between the screening ratio of our SRBO and the additional computational overhead to solve $\bm{\delta}_k$.
	
	We further design an algorithm to address this issue. Define $\Delta_{k+1}$ as the feasible region of $\bm{\delta}_{k+1}$ when the parameter is $\nu_{k+1}$, that is, $\Delta_{k+1} = \{\bm{\delta}_{k+1} | \bm{e^{T}}(\bm{\alpha^{k+1}}+\bm{\delta}_{k+1})\geq \nu_{k+1}, 0\leq\bm{\alpha^{k}} +\bm{\delta}_{k+1} \leq \frac{1}{l}\}$. $\bm{\delta}^{\mathcal{N}}_{k}$ is the partial element of $\bm{\delta}_{k}$ that satisfies $\bm{\delta}_{k} \notin \Delta_{k+1}$, and the corresponding index set is denoted $\mathcal{N}$. $\bm{\delta}^{\mathcal{\bar{N}}}_{k}$ is the partial element of $\bm{\delta}_{k}$ satisfying $\bm{\delta}_{k} \in \Delta_{k+1}$.

	Then, we solve the following smaller optimization problem instead of the original QPP (\ref{deltaQPP}).
	
	\begin{equation}\label{deltakQPP}
		\begin{aligned}
			\bm{\delta}'_{k+1}=\arg \min_{\bm{\delta^{\mathcal{N}}_{k+1}} \in \bar{\Delta}_{k+1}}~r_{k}(\bm{\alpha^k}, \bm{\delta^{\mathcal{\bar{N}}}_{k}},\bm{\delta^{\mathcal{N}}_{k+1}}).
		\end{aligned}
	\end{equation}
	And combine $\bm{\delta'_{k+1}}$ and $\bm{\delta^{\mathcal{N}}_{k+1}}$ as the final $\bm{\delta_{k+1}}$ in the screening process.
	
	The pseudo-code of SRBO-$\nu$-SVM is summarized in Algorithm 1. 

    In summary, our SRBO is embedded in the parameter selection process of $\nu$. We just need to solve the first full optimization problem under the parameter $\nu_{0}$. In the following parameter loop, we use the solution information from the previous step to identify inactive samples. Then, the training computational cost can be greatly reduced. The safety of the solutions is guaranteed.

	\begin{algorithm}[H] 
		\caption{SRBO-$\nu$-SVM}
		\begin{algorithmic}[1] \label{Algorithm1}
			\renewcommand{\algorithmicrequire}{\textbf{Input:}}
			\renewcommand{\algorithmicensure}{\textbf{Output:}}
			\REQUIRE training set $T_{train}$, test data $T_{test}$, one vector of parameter $\nu$, $\bm{\nu}=\{\nu_{0},\nu_{1},\cdots,\nu_{K}\}$\\
			\ENSURE solution $\bm{\alpha^{k}}$, $k=0,1,\cdots,K$, predicted labels\\
			\STATE $k=0$\\
			\STATE $\bm{\alpha^{k}}\leftarrow$ solve QPP (\ref{dualQPPv}) with $T_{train}$ and $\nu_k$\\
			\STATE $\bm{\delta}^{k}\leftarrow$ solve QPP (\ref{deltaQPP}) with $T_{train}$ and $\nu_k$\\		
			\FOR{$k=0:K-1$}
			\STATE $\bm{\delta_{k+1}}\leftarrow$ solve problem (\ref{deltakQPP}) with $\nu_{k+1}$
			\STATE $\rho_{\mathtt{upper}}, \rho_{\mathtt{lower}}\leftarrow$ formula (\ref{rhoUL})
			\STATE $\bm{\alpha^{k+1}_{D}}\leftarrow$ screen from  Corollary \ref{cor-2}
			\STATE $\bm{\alpha^{k+1}_{S}}\leftarrow$ solve reduced QPP (\ref{QPPforAlaphaD}) with $\bm{\delta}_{k+1}$ and $\nu_{k+1}$
			\STATE $\bm{\alpha^{k+1}}\leftarrow$ combine $\bm{\alpha^{k+1}_{D}}$ and $\bm{\alpha^{k+1}_{S}}$
			\STATE Labels $\leftarrow$ predict on $T_{test}$ with each $\alpha^{k}$ according to (\ref{decision2}) 
			\ENDFOR
		\end{algorithmic}
	\end{algorithm}

	\begin{algorithm}[H]
		\caption{DCDM for $\nu$-SVM and SRBO-$\nu$-SVM} 
		\label{alg2} 
		\renewcommand{\algorithmicrequire}{\textbf{Input:}}
		\renewcommand{\algorithmicensure}{\textbf{Output:}}
		\begin{algorithmic}
			\REQUIRE
			matrix $\bm{\bar{Q}}$, 
			parameter $\nu$, 
			size of the optimization problem $\bar{l}$,
			vector $\bm{e}$,
			a threshold $\epsilon$
			\ENSURE 
			dual vector $
   \bm{\alpha}  = 
   (\alpha_{1},\alpha_{2},\cdots,\alpha_{\bar{l}})^{T}
   $
			\STATE Initialize $\bm{\alpha} \in \mathcal{A}_{\nu}$
            \STATE $k = 1$
			\WHILE{$\bm{\alpha}$ is not an $\epsilon$-accurate solution}
			\FOR{$i = 1,2,\cdots,\bar{l}$}
            \STATE $G = \bm{\bar{Q}[i,:]}*\bm{\alpha}$ 
			\IF{$\alpha_{i}-\max(0,\nu-\sum_{k \neq i} \alpha_{k}) < \epsilon $}
			\STATE  $G^{\prime} = \min(G,0)$
			\ELSIF{$\frac{1}{\bar{l}} - \alpha_{i} < \epsilon $}
			\STATE  $G^{\prime} = \max(G,0)$
			\ELSIF{$\max(0,\nu-\sum_{k \neq i} \alpha_{k}) < \alpha_{i} < \frac{1}{\bar{l}}$}
			\STATE  $G^{\prime} = G$
			\ENDIF
			\IF{$ \left|G^{\prime}-0\right| > \epsilon $}
			\STATE $\alpha_{i} = \min(\frac{1}{\bar{l}},\max(\alpha_{i}-G^{\prime}*s,0,\nu-\sum_{k \neq i} \alpha_{k})) $ 
			\ENDIF
			\ENDFOR
                \STATE $k = k+1$
			\ENDWHILE
		\end{algorithmic} 
	\end{algorithm}

	\subsection{The Algorithm for DCDM}\label{DCDMAlgorithm}

 It is easy to find that our acceleration method of SRBO is independent from the solvers of QPP. That is, the solver will not have an effect on our safe screening rule. Furthermore, we develop an efficient dual coordinate descent method (DCDM) for solving QPP of $\nu$-SVM and SRBO-$\nu$-SVM. The pseudocode of DCDM for $\nu$-SVM and SRBO-$\nu$-SVM is summarized in Algorithm 2, in which $\bm{\bar{Q}} \in \mathbb{R}^{\bar{l}\times \bar{l}} $ is the kernel matrix of input.
	
	The DCDM is independent of the screening process. The dual objective is updated by completely solving for one coordinate while keeping all other coordinates fixed. It is simple and reaches an $\epsilon$-accurate solution in $O$(log($1$/$\epsilon$)) iterations \cite{hsieh_dual_2008}. Note that an $\epsilon$-accurate solution $\bm{\alpha}$ is defined if $F(\bm{\alpha})\leq F(\bm{\alpha}^{*}) + \epsilon $. Therefore, the total time complexity of DCDM will not exceed $O$($l$ log($1$/$\epsilon$)). On the contrary, if we use ``quadprog'' toolbox of MATLAB to solve the dual problem (\ref{dualQPPv}), the time complexity is $O(l^3)$. Especially when $l$ is large, DCDM will have more computational advantages. 		
	
	\section{A General Discussion on the Proposed SRBO} \label{GeneralDiscussion}

	We could rewrite the primal problem of SVM-type models in the following unified formulation.
	\begin{equation}\label{GeneralForm}
		\min \frac{1}{2}\|\bm{w}\|_{2}^{2}+C \cdot L(h, \rho)-v \rho.
	\end{equation}
	For the supervised $C$-SVM and $\nu$-SVM, the decision boundary $h$ is 
	$$
	h(\bm{w};\bm{x},y) = y \langle	\bm{w},\Phi(\bm{x})\rangle = 0.
	$$
	The loss $L(h(\bm{w};\bm{x},y), \rho)$ is defined as a hinge function.
	\begin{equation}\label{HingeLoss}
		L(h,  \rho) = \sum_{i=1}^{l} max\{0, \rho- h\}.
	\end{equation}
	Specifically, the classical $C$-SVM corresponds to the case where $\nu = 0$, $\rho = 1$. In contrast, in $\nu$-SVM, $\nu$ is a manual parameter selected from the interval $(0, 1)$ and $\rho$ is a variable to be solved.
	
	From this point of view, \cite{pmlr-v32-wangd14} provides a screening approach for formulation (\ref{GeneralForm}) when $\nu = 0$ and $\rho = 1$. In comparison, our screening method provides a more general rule for the case that $\rho$ is a variable rather than a fixed number. The breakthrough point is that we provide the upper and lower bounds for the optimal $\rho^{*}$ in Corollary \ref{corollary2} .

	\newcolumntype{L}[1]{>{\raggedright\arraybackslash}m{#1}}
	\newcolumntype{C}[1]{>{\centering\arraybackslash}m{#1}}
	\newcolumntype{R}[1]{>{\raggedleft\arraybackslash}m{#1}}
	
	\begin{table*}[htbp!]
		\centering
		\caption{The Main Formulations of SSR-$\nu$-SVM and SSR-OC-SVM}
      	\renewcommand{\arraystretch}{0.1}
		\resizebox{0.9\linewidth}{!}{
		\begin{tabular}{m{2cm}m{7.3cm}m{7cm}}
			\toprule
			&\multicolumn{1}{c}{ $\nu$-SVM} & \multicolumn{1}{c}{ OC-SVM} \\
			\midrule
			Primal Problem & 
			\begin{eqnarray*}
				\displaystyle{\min_{\bm{w},\xi,\rho}}~~&& \frac{1}{2}\|\bm{w}\|^{2}-\nu\rho+\frac{1}{l}\sum_{i=1}^{l}\xi_{i}\\
				\mbox{s.t.}~~&&y_i \bm{w}\cdot \Phi(\bm{x_{i}}) \geq \rho -\xi_{i} \nonumber\\
				~~&&\xi_{i}\geq 0, i=1,2,\cdots,l \nonumber\\
				~~&&\rho \geq 0 \nonumber
			\end{eqnarray*}
			&
			\begin{eqnarray*}
				\displaystyle{\min_{\bm{w},\xi,\rho}}~~&& \frac{1}{2}\|\bm{w}\|^{2}-\rho+\frac{1}{\nu l}\sum_{i=1}^{l}\xi_{i}\\
				\mbox{s.t.}~~&&(\bm{w}\cdot\Phi(\bm{x_{i}})) \geq \rho -\xi_{i},\nonumber\\
				~~&&\xi_{i}\geq 0, i=1,2,\cdots,l \nonumber
			\end{eqnarray*}\\
			\midrule
			Dual Problem & 
			\begin{eqnarray*}
				\displaystyle{\min_{\alpha}}&&\frac{1}{2}\alpha^{T}Q\alpha\\
				\mbox{s.t.}	&&\bm{e^{T}}\alpha\geq \nu \nonumber\\
				&&0\leq\alpha\leq\frac{1}{l} \nonumber
			\end{eqnarray*}
			~~~~~~~~~~~~~~where $Q_{ij}=y_{i}\kappa(\bm{x_{i}},x_{j})y_{j}$.
			&
			\begin{eqnarray*}
				\displaystyle{\min_{\alpha}}&&\frac{1}{2}\alpha^{T}H\alpha\\
				\mbox{s.t.}&&\bm{e^{T}}\alpha=1,\nonumber\\
				&&0\leq\alpha \leq\frac{1}{\nu l} \nonumber
			\end{eqnarray*}
			~~~~~~~~~~~~~~where $H_{ij}=\kappa(\bm{x_{i}},x_{j})$.\\
			\midrule
			Sparsity of Solution & 
			\begin{eqnarray*}
				y_{i}\bm{w^*}\cdot \Phi(\bm{x_{i}})>\rho^{*} \Longrightarrow \alpha_{i}^{*}=0\\
				y_{i}\bm{w^*}\cdot \Phi(\bm{x_{i}})<\rho^{*} \Longrightarrow \alpha_{i}^{*}=\frac{1}{l}
			\end{eqnarray*}
			& 
			\begin{eqnarray*}
				\bm{w^*}\cdot \Phi(\bm{x_{i}})>\rho^{*} \Longrightarrow \alpha_{i}^{*}=0 \\
				\bm{w^*}\cdot \Phi(\bm{x_{i}})<\rho^{*} \Longrightarrow \alpha_{i}^{*}=\frac{1}{\nu l}	
			\end{eqnarray*}	
			\\
			\midrule
			Screening rule
			&
			{\small
				\begin{equation*}
					\begin{aligned}
						&& \ & r_{k} \cdot y_{i}\Phi(\bm{x_{i}})-| r_{k}|^{\frac{1}{2}}\|y_{i}\Phi(\bm{x_{i}})\| >d_{\mathtt{upper}}(\lfloor l-\nu_{k} l\rfloor) \\  
						&&  \ &\Longrightarrow\alpha^{k+1}_{i}=0.
					\end{aligned}
			\end{equation*}	}			
			{\small
				\begin{equation*}
					\begin{aligned}
						&& \ & r_{k} \cdot y_{i}\Phi(\bm{x_{i}}) +| r_{k}|^{\frac{1}{2}}\|y_{i}\Phi(\bm{x_{i}})\| <d_{\mathtt{lower}}(\lceil l-\nu_{k} l\rceil)\\
						&&  \ & \Longrightarrow	\alpha^{k+1}_{i}=\frac{1}{l},
					\end{aligned}
			\end{equation*} }			
			& 
			{\small
				\begin{equation*}
					\begin{aligned}
						&&\ & r_{k} \cdot \Phi(\bm{x_{i}}) -| r_{k}|^{\frac{1}{2}}\|\Phi(\bm{x_{i}})\|
						>d_{\mathtt{upper}}(\lfloor l-\nu_{k} l\rfloor)\\
						&&\ & \Longrightarrow \alpha^{k+1}_{i}=0. 
					\end{aligned}
			\end{equation*} }
			
			{\small
				\begin{equation*}
					\begin{aligned}
						&& \ &r_{k}\cdot\Phi(\bm{x_{i}})+| r_{k}|^{\frac{1}{2}}\|\Phi(\bm{x_{i}})\|<d_{\mathtt{lower}}(\lceil l-\nu_{k} l\rceil)\\  
						&&\ & \Longrightarrow \alpha^{k+1}_{i}=\frac{1}{\nu_{k+1}l}.  
					\end{aligned}
			\end{equation*} }
			\\			
			\bottomrule
		\end{tabular}}
		\label{tab:addlabel}
	\end{table*}

	Furthermore, based on this unified formulation, we can easily give the screening rule for unsupervised OC-SVM \cite{scholkopfEstimatingSupportHighDimensional2001}. Since in OC-SVM the decision boundary $h$ is 
	$$
	h(\bm{w};\bm{x}) = \langle	\bm{w},\Phi(\bm{x})\rangle = 0.
	$$
	And the loss function $L(h(\bm{w};\bm{x}), \rho)$ is also a hinge function given in (\ref{HingeLoss}). The only difference from $\nu$-SVM is the absence of $y_{i}$ in OC-SVM. 
 
	For the sake of brevity, the derivation of the SRBO of OC-SVM is omitted. We directly provide the result of the SRBO of the OC-SVM in Table \ref{tab:addlabel}.

	\section{Numerical Experiments}\label{experiment}
	To verify the advantages of our proposed method, we conduct numerical experiments for supervised and unsupervised tasks. The experimental data sets consist of 6 artificial and 30 real-world benchmark data sets. Our codes are available on the web. \footnote{\url{https://github.com/Citrus-Gradenia/safe-screening-for-nu-svm/tree/main/nu-svm}}. 
	
	The 6 artificial data sets include normal distributions with three different means, circular shape, exclusive case, and spiral case (shown in Fig. \ref{fignusvmSimplie}), respectively. 
	
	Among the real-world benchmark data sets, 29 data sets are downloaded from the machine learning data set repository at the University of California, Irvine (UCI) \cite{Lichman:2013} or the web of LIBSVM Data\footnote{https://www.csie.ntu.edu.tw/$\thicksim$cjlin/libsvmtools/datasets/}. Another one is the MNIST data set, which comes from the National Institute of Standards and Technology (NIST) in the United States. Their statistics are shown in Table \ref{Table222}. The original 30 data sets, except for MNIST data, are binary classification data. In addition, if the test sets are not provided, we will use four-fifths of the random samples for training and the other fifth for test.
	\setlength{\tabcolsep}{4pt}
	\begin{table}[htbp!]
		\centering
		\caption{The statistics of 30 benchmark data sets. }\label{Table222}
        \resizebox{0.9\textwidth}{!}{
		\begin{tabular}{ccccc|ccccc}\hline
			Data set &$^\#$Instances &$^\#$Positive&$^\#$Negative&$^\#$Features&Data set &$^\#$Instances &$^\#$Positive&$^\#$Negative&$^\#$Features \\
			\noalign{\smallskip} \hline \noalign{\smallskip}
			Hepatitis
			&80&13&67&19&CMC
			&1473&629&844&9\\
			Fertility
			&100&88&12&9&Yeast
			&1484&463&1021&9\\
			Planning Relax
			&146&130&52&12&Wifi-localization
			&2000&500&1500&9\\
			Sonar
			&208&97&111&60&CTG
			&2126&1655&471&22\\
			SpectHeart
			&267&212&55&44&Abalone
			&4177&689&3488&8\\
			Haberman
			&306&225&81&3&Winequality
			&4898&1060&3838&11\\
			LiverDisorder
			&345&145&200&6&ShillBidding
			&6321&5646&675&10\\
			Monks
			&432&216&216&6&Musk
			&6598&5581&1017&166\\
			BreastCancer569
			&569&357&212&30&Electrical
			&10000&3620&6380&13\\
			BreastCancer683
			&683&444&239&9&Epiletic
			&11500&2300&9200&178\\ 
			Australian
			&690&307&383&14&Nursery
			&12960&8640&4320&8\\
			Pima
			&768&500&268&8&credit card
			&30000&6636&23364&23\\
			Biodegration
			&1055&356&699&41&Accelerometer
			&31991&31420&571&6\\
			Banknote
			&1372&762&610&4&Adult
			&32561&7841&24720&14\\
			HCV-Egy
			&1385&362&1023&28
   & MNIST & 70000 & - & - & 28$\times$28\\
			\bottomrule
		\end{tabular}}
	\end{table}

	 For binary classification case, the total accuracy on the test set is used as the evaluation criterion. For one class case, only positive training samples are used as a training set, and all positive and negative test samples are used to evaluate the prediction area under the curve (AUC) of the models. To measure the computational efficiency, the training time for each algorithm is also provided. 
	
	\begin{figure*}
		\centering
		\subfloat[]
		{\includegraphics[width=0.18\textwidth]{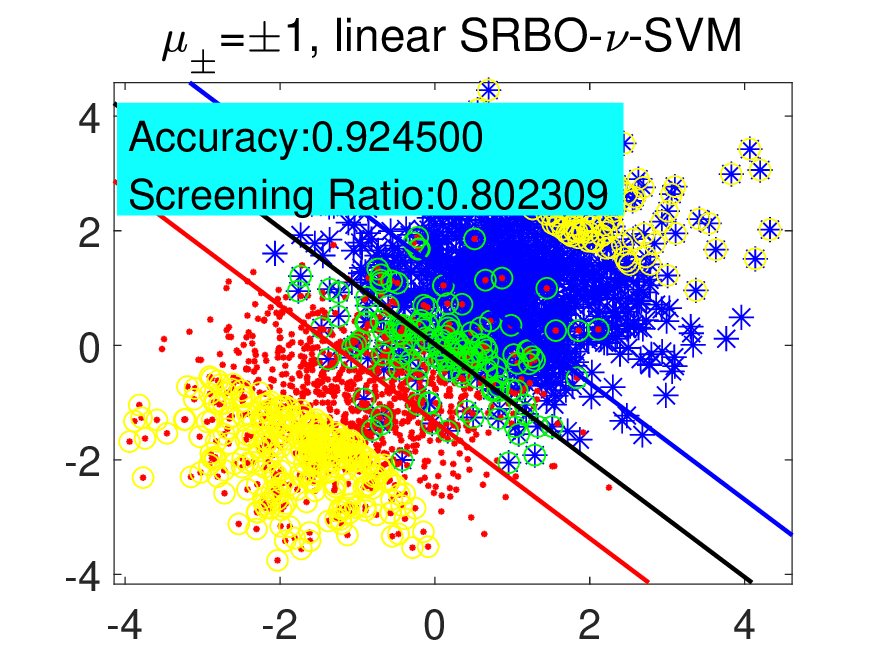}}
		\subfloat[]
		{\includegraphics[width=0.18\textwidth]{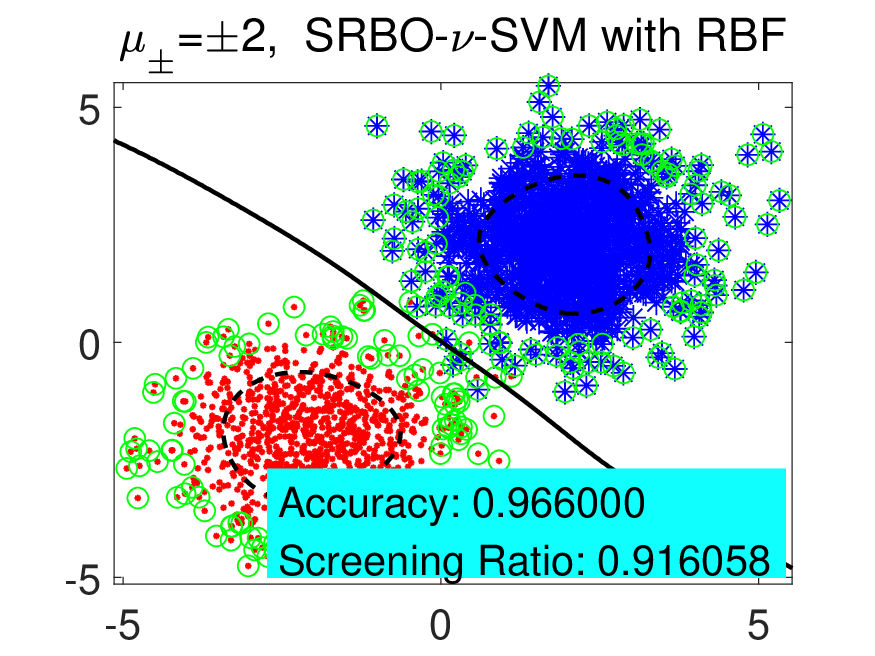}}
		\subfloat[]
		{\includegraphics[width=0.18\textwidth]{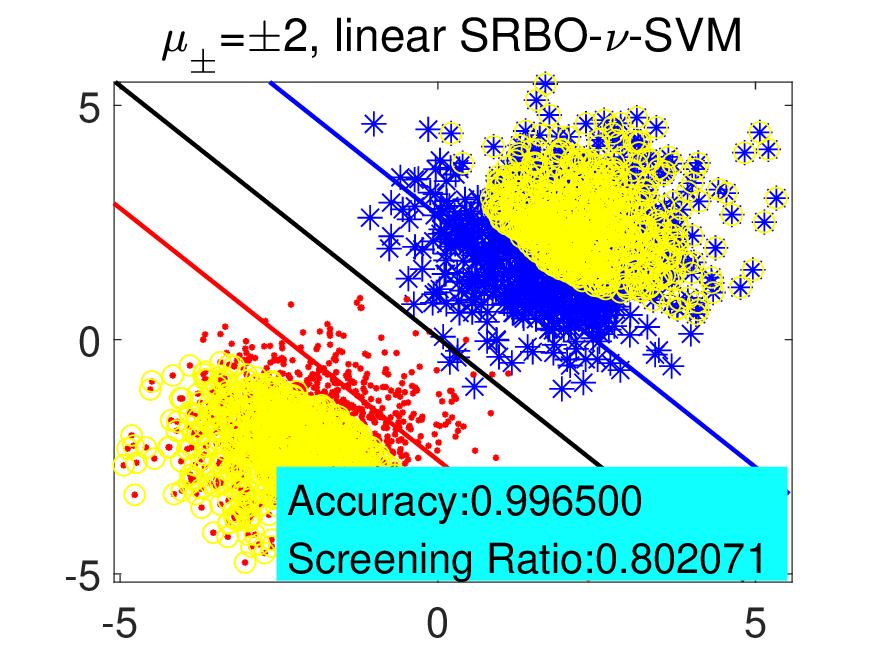}}
		\subfloat[]
		{\includegraphics[width=0.18\textwidth]{Simple2-screening_NU_SVM-rbf-194.eps}} 
		\subfloat[]
		{\includegraphics[width=0.18\textwidth]{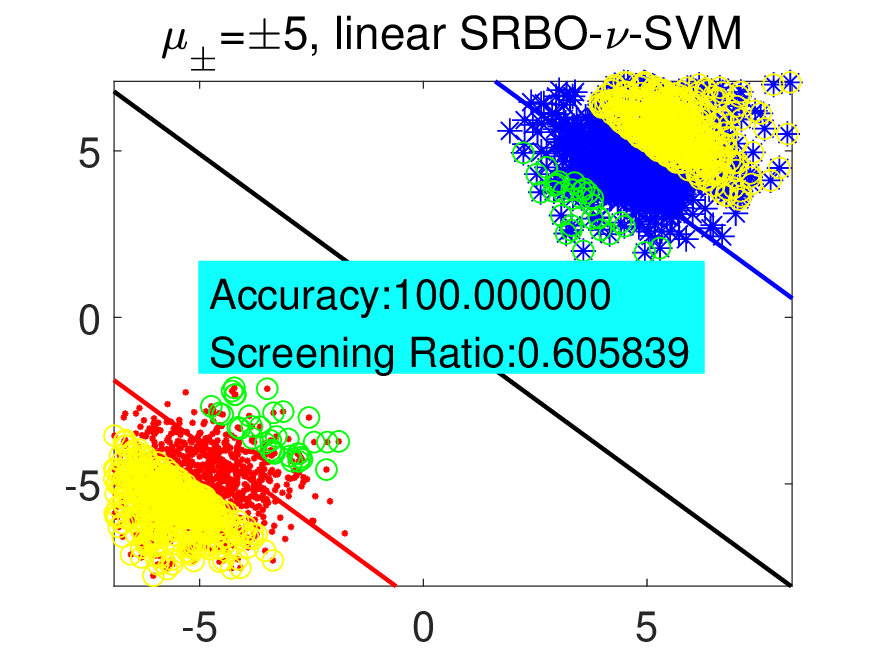}}
		
		\subfloat[]
		{\includegraphics[width=0.18\textwidth]{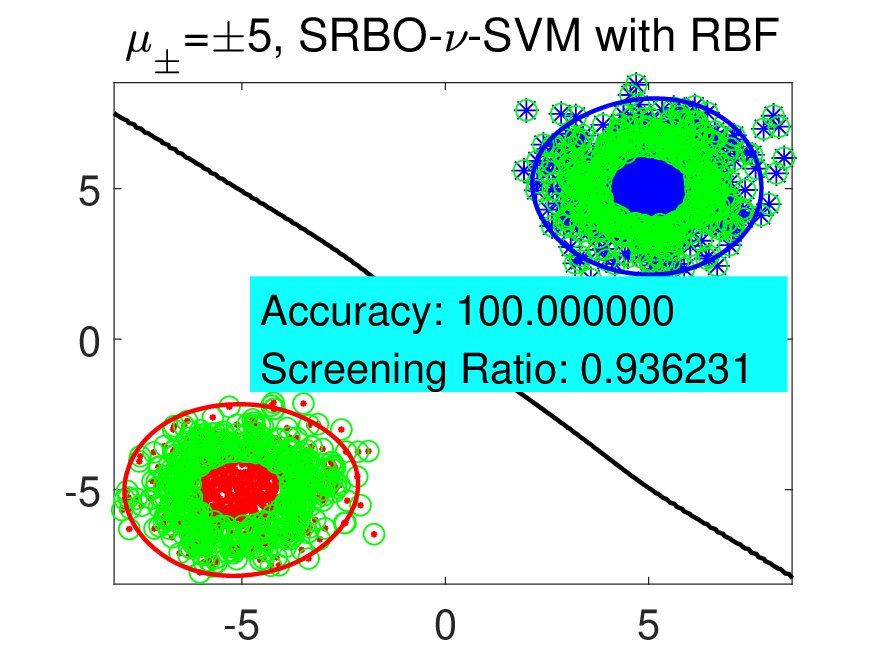}}
		\subfloat[]
		{\includegraphics[width=0.18\textwidth]{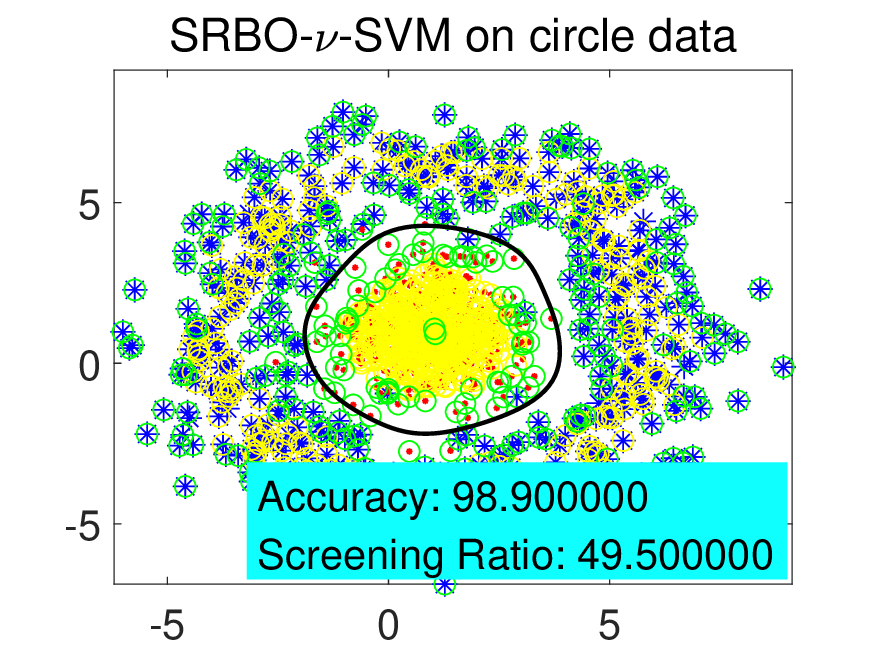}}
		\subfloat[]
		{\includegraphics[width=0.18\textwidth]{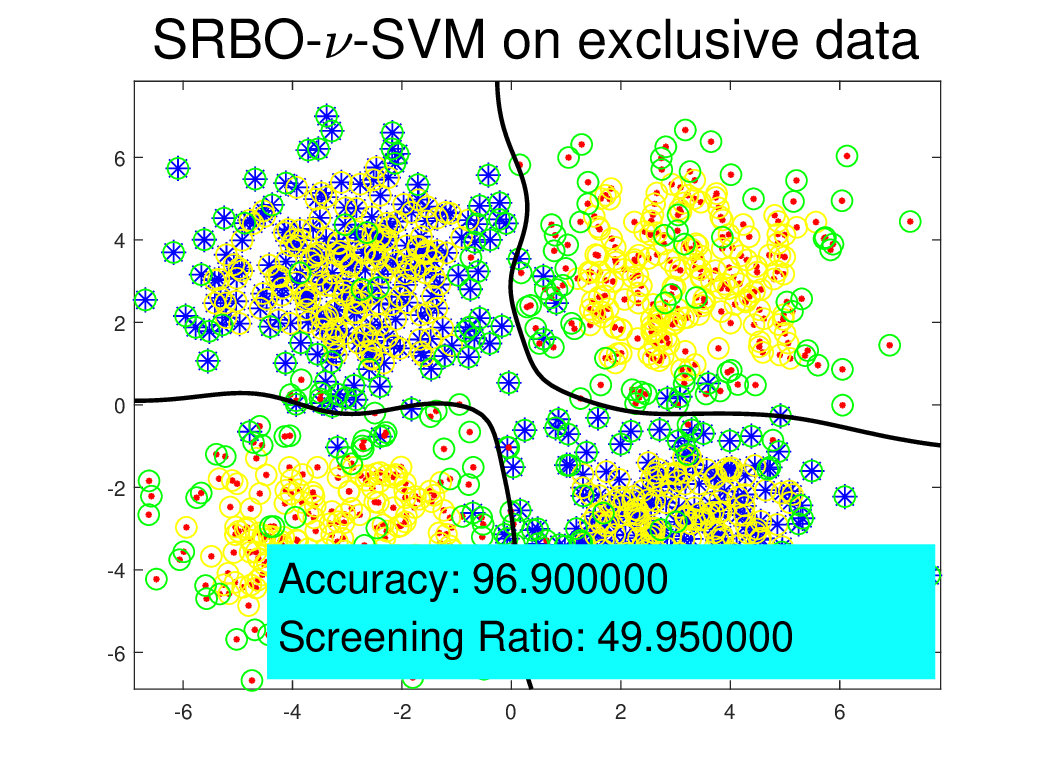}}
		\subfloat[]
		{\includegraphics[width=0.18\textwidth]{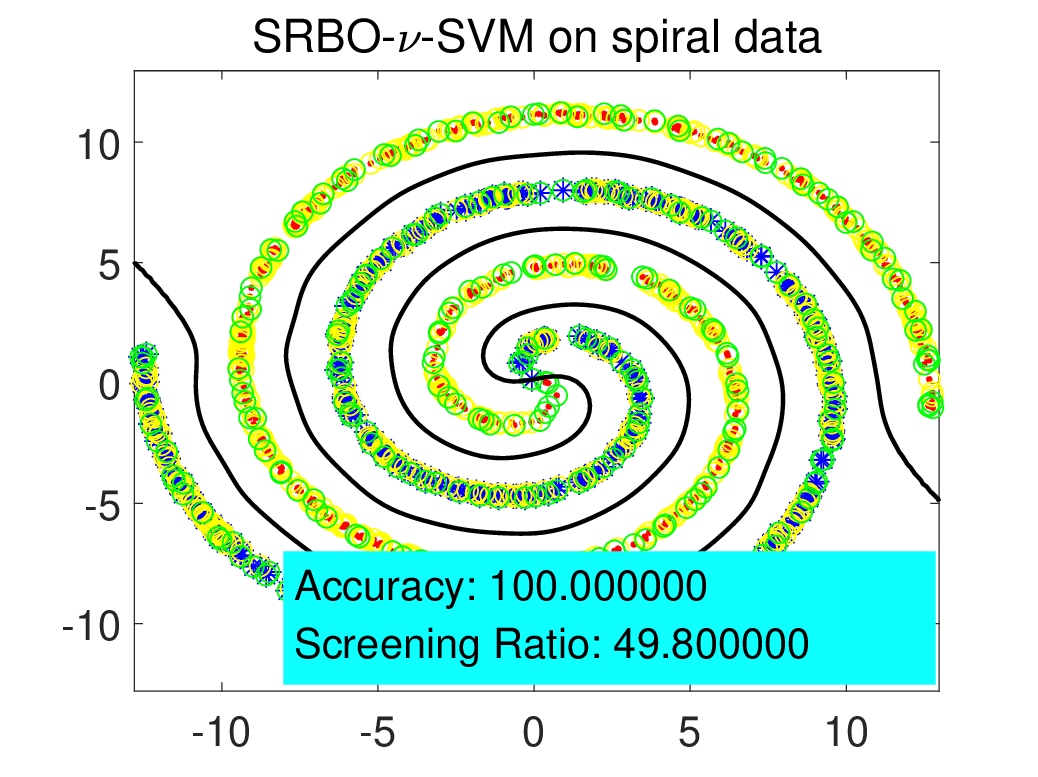}}
		\caption{Classification graphs of SRBO-$\nu$-SVM on three normally distributed data sets (respectively in linear and nonlinear case, $\mu_{\pm}=\pm 1, \pm 2, \pm 5$), nonlinear case on circle, exclusive and spiral data. The blue and red points represent the positive and negative training instances, respectively. In each graph, the black solid line is the decision boundary, and the other two lines are support hyperplanes. Each graph corresponds to the classifier under optimal parameters, and `Accuracy' represents the corresponding training accuracy. The green points correspond to the samples deleted in $\mathcal{L}$, and the yellow points correspond to the samples removed in $\mathcal{R}$. `Screening Ratio' is the average result during the whole parameter selection process by SRBO.}
		\label{fignusvmSimplie}
	\end{figure*}

       \setlength{\belowcaptionskip}{-5cm}  
	All experiments are implemented on Windows 10 platform with MATLAB R2018b. The computer configuration is an Intel (R) Core (TM) i5-6200U CPU @ 2.30GHz 8GB. For a fair comparison, in Sections 5.1 and 5.2, all QPPs are solved by the MATLAB toolbox `quadprog' with the default setting `interior-point-convex'. The efficiency of our proposed DCDM solver is verified in Sections 5.4 and 5.5.

	For the kernel function in SVM-type methods, both the linear kernel $\kappa(\bm{x_i}, \bm{x_j})=(\bm{x_i}\cdot \bm{x_j})$ and the nonlinear radial basis function (RBF) $\kappa(\bm{x_i}, \bm{x_j})=exp(- \|\bm{x_i}-\bm{x_j}\|^2/2\sigma^2)$ are considered.
	
	All manual parameters are selected through the grid search approach. The parameter $\nu$ is selected from ($0.01:0.001:1-1/l$), and the parameter of RBF $\sigma$ in each algorithm is selected over the range $\{2^i|i=-3,-2,\cdots,8\}$. 
	
	\subsection{Experiments on Supervised Models}  \label{Experiment: nu-SVM}
	First, we verify the feasibility and efficiency of our supervised SRBO-$\nu$-SVM on artificial data sets and 26 small-scale benchmark data sets (sample sizes do not exceed 13,000).
	
	\textbf{Experiments on 6 Artificial Data Sets. }
	Figs. \ref{fignusvmSimplie}a - \ref{fignusvmSimplie}f show the results of our SRBO-$\nu$-SVM on three normally distributed data sets. `Accuracy' represents the prediction accuracy under optimal parameters in each case. `Screening Ratio' denotes the average percentage of screening (\%) by SRBO (corresponding to the proportion of deleted samples in $\mathcal{L}$ and $\mathcal{R}$ for each parameter). Each data set contains two classes of samples and each class has 1000 data points, from $N(\{\mu_{+},\mu_{-}\}^{T},I)$, where $I\in\mathbb{R}^{2\times2}$ is the identity matrix. For the positive class, $\mu_{+}=1, 2, 5$, and for the negative class, $\mu_{-}=-1, -2, -5$, respectively. Figs. \ref{fignusvmSimplie}g - \ref{fignusvmSimplie}i give the results on three other data sets: circle data, exclusive data, and spiral data. For these three data sets, each class contains 500 data points.	
			
	As shown in Fig. \ref{fignusvmSimplie}, the SRBO-$\nu$-SVM screens out most inactive instances and remains a few points for training. But the original $\nu$-SVM uses all training samples to build the classifier. More importantly, the prediction accuracy remains unchanged with the screening rule. It implies the effectiveness and safety of our SRBO-$\nu$-SVM. 
	
	\textbf{Experiments on Benchmark Data Sets. }  \label{nu-SVM: benchmark}
	The comparison performances of three methods $C$-SVM, original $\nu$-SVM, and SRBO-$\nu$-SVM with linear and RBF kernel are shown in Table \ref{Table3} and \ref{Table4}. `Time' provides the average time (in seconds) of the training for each parameter. To clearly illustrate the effectiveness of our SRBO-$\nu$-SVM, the `Speedup Ratio' is defined as
	\begin{eqnarray}
		Speedup~Ratio=\frac{Time~of~\nu-SVM }{Time~of~SRBO-\nu-SVM}
	\end{eqnarray}
    We just use larger 13 data sets in linear case, because the calculations on small-scale data in linear case is very small as it is. Then it makes no sense to do the acceleration. In addition, `Win/Draw/Loss' represents the corresponding performance comparisons between our proposed SRBO method and the competitors in prediction accuracy and computational time.
	
	From Tables \ref{Table3} and \ref{Table4}, the training time of our SRBO-$\nu$-SVM is significantly shorter in both linear and nonlinear cases. Especially in the nonlinear case, the computational advantage of SRBO-$\nu$-SVM is more obvious, which demonstrates that our proposed safe screening rule shows superior performance. When the sample size is small, the original $\nu$-SVM seems to run faster, but when the sample size exceeds 500, the advantages of our SRBO-$\nu$-SVM are gradually revealed. When the sample size is tens of thousands, the advantages of our SRBO-$\nu$-SVM are more obvious. It illustrates that our SRBO can reduce computing costs, especially when dealing with large-scale data. The accuracy of $\nu$-SVM is better than that of $C$-SVM on more than half of the data sets. Furthermore, the prediction accuracies of $\nu$-SVM and SRBO-$\nu$-SVM are always the same, demonstrating the safety of our SRBO.

\begin{table*}[htbp!]
	\renewcommand\arraystretch{1.1}
    \centering
	\resizebox{0.8\textwidth}{!}{
	\begin{threeparttable}[b]
		\centering
		\caption{Comparisons of 3 kinds of supervised SVMs on 13 benchmark data sets in linear case.}\label{Table3}
		\begin{tabular}{*{9}{c}}
			\toprule
			\multirow{2}{*}{Data set}&\multicolumn{2}{c}{$C$-SVM}&\multicolumn{2}{c}{$\nu$-SVM}&\multicolumn{4}{c}{SRBO-$\nu$-SVM}\\
			\cmidrule(lr){2-3}\cmidrule(lr){4-5}\cmidrule(lr){6-9}
			&Accuracy(\%)&Time(s)&Accuracy(\%)&Time(s)&Accuracy(\%)&Time(s)&Screening Ratio(\%)&Speedup Ratio     \\
			\midrule
			Banknote
			&98.18&\bf{0.2825}&\bf{98.91}&0.3918&\bf{98.91}&0.5801&9.17&0.6754\\
			HCV-Egy
			&\bf{73.65}&\bf{0.1869}&\bf{73.65}&0.5140&\bf{73.65}&0.4135&19.46&\bf{1.2430}\\
			CMC
			&62.93&\bf{0.3525}&\bf{68.37}&0.5626&\bf{68.37}&1.1268&8.83&0.4993\\
			Yeast
			&69.02&\bf{0.3579}&\bf{69.36}&0.3647&\bf{69.36}&0.4771&22.93&0.7644\\
			Wifi-localization
			&\bf{99.50}&\bf{0.6528}&\bf{99.50}&1.4750&\bf{99.50}&1.7274&48.21&\bf{1.7339}\\
			CTG
			&96.00&\bf{0.8130}&\bf{96.47}&0.9797&\bf{96.47}&1.3249&15.25&0.7394\\
			Abalone
			&\bf{83.47}&3.3469&\bf{83.47}&10.2277&\bf{83.47}&\bf{1.1584}&73.47&\bf{8.8293}\\
			Winequality
			&\bf{78.37}&6.8799&\bf{78.37}&9.2887&\bf{78.37}&\bf{5.0118}&70.92&\bf{1.8534}\\
			ShillBidding 
			&\bf{98.26}&22.6100&98.10&18.7526&98.10&\bf{2.5058}&76.54&\bf{7.4837}\\
			Musk
			&\bf{94.31}&19.6065&93.93&18.6710&93.93&\bf{17.6803}&45.79&\bf{1.0560}\\
			Electrical
			&99.45&50.9515&\bf{99.65}&15.8504&\bf{99.65}&\bf{3.7544}&52.73&\bf{4.2218}\\
			Epiletic
			&80.61&132.7730&\bf{81.43}&12.1297&\bf{81.43}&\bf{7.0235}&10.48&\bf{1.7270}\\ 
			Nursery
			&85.29&150.6635&\bf{100.00}&16.6886&\bf{100.00}&\bf{5.5335}&51.62&\bf{3.0159}\\
           \midrule
   	  Win/Draw/Loss&7/4/2&&0/13/0&&&\\
            Win/Draw/Loss&&7/0/6&&9/0/4&\\
			\bottomrule
		\end{tabular}
	\end{threeparttable}}
\end{table*}

\begin{table*}[htbp!]
	\renewcommand\arraystretch{1.1}
    \centering
	\resizebox{0.8\textwidth}{!}{
	\begin{threeparttable}[b]
		\centering
		\caption{Comparisons of 3 kinds of supervised SVMs on 26 benchmark data sets in nonlinear case.}\label{Table4}
		\begin{tabular}{*{9}{c}}
			\toprule
			\multirow{2}{*}{Data set}&\multicolumn{2}{c}{$C$-SVM}&\multicolumn{2}{c}{$\nu$-SVM}&\multicolumn{4}{c}{SRBO-$\nu$-SVM}\\
			\cmidrule(lr){2-3}\cmidrule(lr){4-5}\cmidrule(lr){6-9}
			&Accuracy(\%)&Time(s)&Accuracy(\%)&Time(s)&Accuracy(\%)&Time(s)&Screening Ratio(\%)&Speedup Ratio\\
			\midrule
			Hepatitis
			&\bf{93.33}&\bf{0.0046}&86.67&0.0103&86.67&0.0211&15.50&0.4852\\
			Fertility
			&\bf{90.00}&0.0031&\bf{90.00}&0.0082&\bf{90.00}&0.0101&22.50&0.8107\\
			Planning Relax
			&\bf{72.22}&0.0123&\bf{72.22}&0.0154&\bf{72.22}&0.0176&22.46&0.8760\\
			Sonar
			&\bf{83.33}&\bf{0.0129}&80.95&0.0108&80.95&0.0183&2.60&0.5912\\
			SpectHeart
			&79.63&0.0161&\bf{85.19}&\bf{0.0135}&\bf{85.19}&0.0144&17.67&0.9340\\
			Haberman
			&\bf{80.33}&0.0272&\bf{80.33}&\bf{0.0186}&\bf{80.33}&0.0204&21.47&0.9126\\
			LiverDisorder
			&57.97&0.0266&\bf{71.01}&\bf{0.0151}&\bf{71.01}&0.0175&11.58&0.8628\\
			Monks
			&86.21&0.0294&\bf{95.40}&\bf{0.0263}&\bf{95.40}&0.0300&6.09&0.8756\\
			BreastCancer569
			&97.35&0.0485&\bf{99.12}&0.0471&\bf{99.12}&\bf{0.0438}&8.65&\bf{1.0796}\\
			BreastCancer683
			&95.59&0.0796&\bf{97.06}&0.0722&\bf{97.06}&\bf{0.0496}&14.83&\bf{1.4571}\\
			Australian
			&\bf{88.49}&0.0862&87.77&0.0702&87.77&\bf{0.0590}&8.52&\bf{1.1907}\\
			Pima
			&75.82&0.0950&\bf{76.47}&0.0907&\bf{76.47}&\bf{0.0691}&12.66&\bf{1.3133}\\
			Biodegration
			&86.26&\bf{0.1578}&\bf{91.00}&0.2257&\bf{91.00}&0.1650&21.49&\bf{1.3676}\\
			Banknote
			&98.55&\bf{0.3042}&\bf{100.00}&0.5184&\bf{100.00}&0.3874&11.05&\bf{1.3381}\\
			HCV-Egy
			&\bf{73.65}&0.2944&\bf{73.65}&0.3072&\bf{73.65}&\bf{0.2516}&5.69&\bf{1.2210}\\
			CMC
			&64.97&0.3404&\bf{71.09}&0.3631&\bf{71.09}&\bf{0.2991}&4.87&\bf{1.2140}\\
			Yeast
			&\bf{74.07}&0.3378&73.06&0.4074&73.06&\bf{0.3038}&9.09&\bf{1.3410}\\
			Wifi-localization
			&\bf{99.50}&0.6397&\bf{99.50}&0.9700&\bf{99.50}&\bf{0.5842}&12.25&\bf{1.6604}\\
			CTG
			&\bf{97.88}&0.8131&\bf{97.88}&1.0807&\bf{97.88}&\bf{0.7319}&29.84&\bf{1.4765}\\
			Abalone
			&83.47&7.8309&\bf{84.07}&8.6019&\bf{84.07}&\bf{6.7924}&0.79&\bf{1.2664}\\
			Winequality
			&\bf{79.18}&11.8288&78.37&10.1373&78.37&\bf{3.8114}&14.13&\bf{2.6957}\\
			ShillBidding 
			&98.42&17.5466&\bf{98.81}&11.8200&\bf{98.81}&\bf{3.5394}&20.17&\bf{3.3395}\\
			Musk
			&90.98&19.8758&\bf{98.26}&11.6853&\bf{98.26}&\bf{6.6632}&10.52&\bf{1.7537}\\
			Electrical
			&98.65&48.1828&\bf{98.95}&13.6134&\bf{98.95}&\bf{4.1703}&38.30&\bf{3.2643}\\
			Epiletic
			&94.57&100.8499&\bf{96.70}&19.5876&\bf{96.70}&\bf{9.2757}&11.86&\bf{2.1117}\\ 
			Nursery
			&\bf{100.00}&100.8287&\bf{100.00}&22.2967&\bf{100.00}&\bf{9.5264}&38.12&\bf{2.3405}\\
            \midrule
   	  Win/Draw/Loss&14/7/5&&0/26/0&&&\\
            Win/Draw/Loss&&19/0/7&&18/0/8&\\
			\bottomrule
		\end{tabular}
	\end{threeparttable}}
\end{table*}

    \begin{figure}[htbp!]
		\centering
		\includegraphics[width=0.5\textwidth]{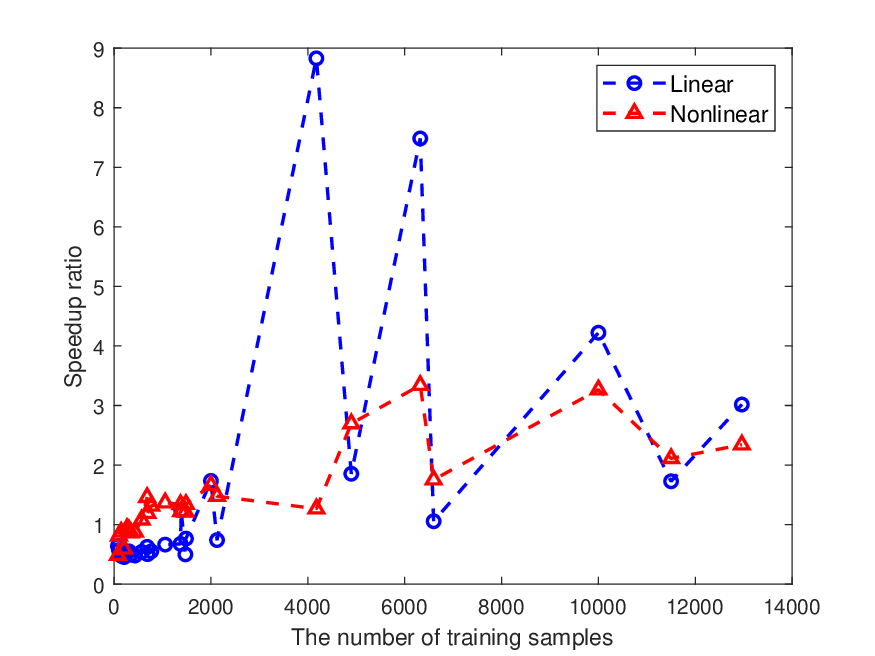}
		\caption{Speedup Ratio  of SRBO-$\nu$-SVM in linear and nonlinear cases.}
		\label{fig2}
		\vspace{-1.5em}
	\end{figure}

	The statistical results on the `Speedup Ratio' of our SRBO-$\nu$-SVM for the benchmark data sets are shown in Fig. \ref{fig2}. The blue and red lines correspond to linear and nonlinear cases, respectively. In general, with the sample size gradually increasing, the speedup ratio increases in both linear and nonlinear cases. It implies that the advantage of our SRBO is more significant for large-scale data. For the nonlinear case, when the sample size is too large, the acceleration effect is slightly limited. The main reason is that the computational overhead of the extra RBF matrix is relatively high.
	\begin{figure*}
		\centering
		\subfloat[]
		{\includegraphics[width=0.25\textwidth]{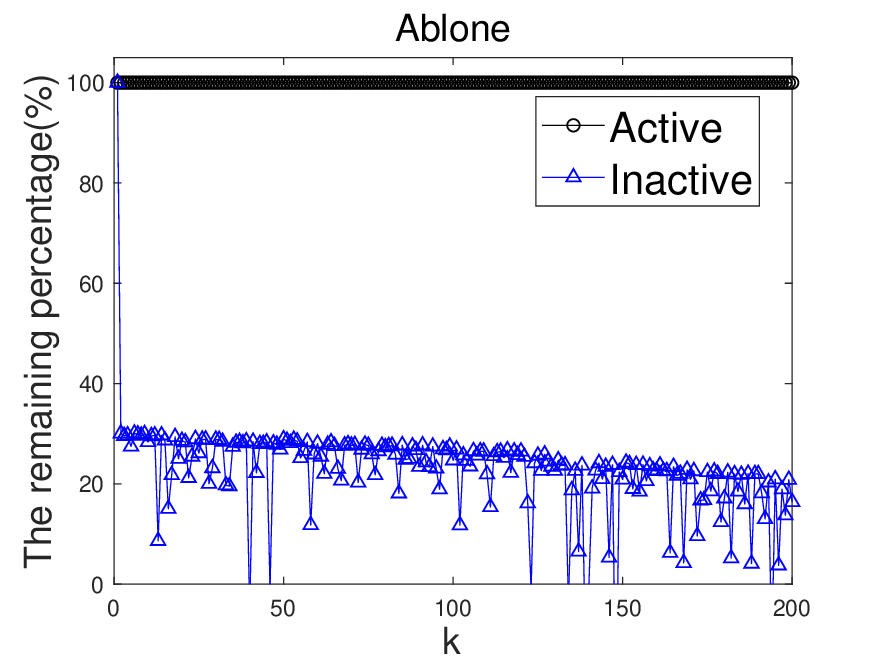}}
		\subfloat[]
		{\includegraphics[width=0.25\textwidth]{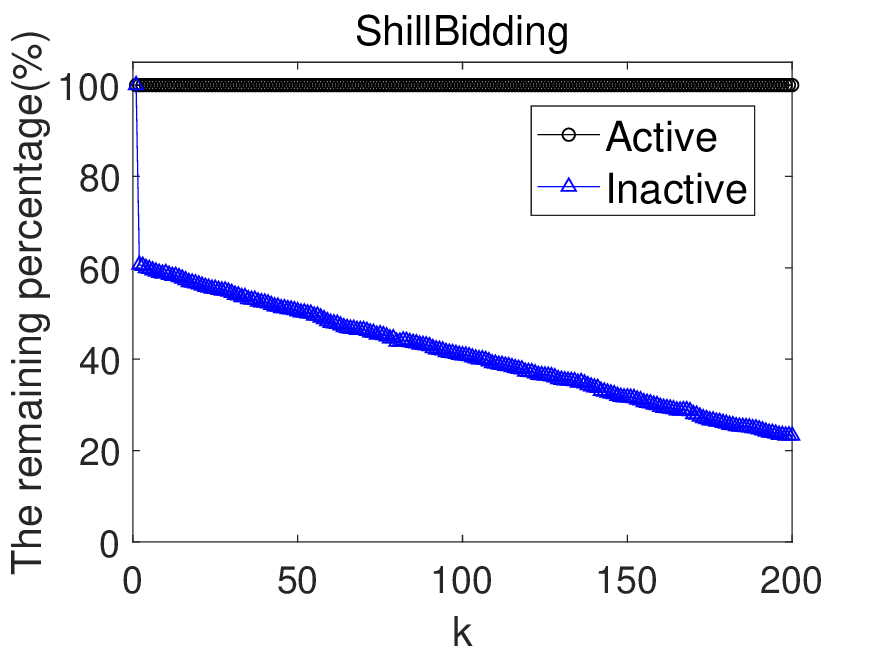}}
		\subfloat[]
		{\includegraphics[width=0.25\textwidth]{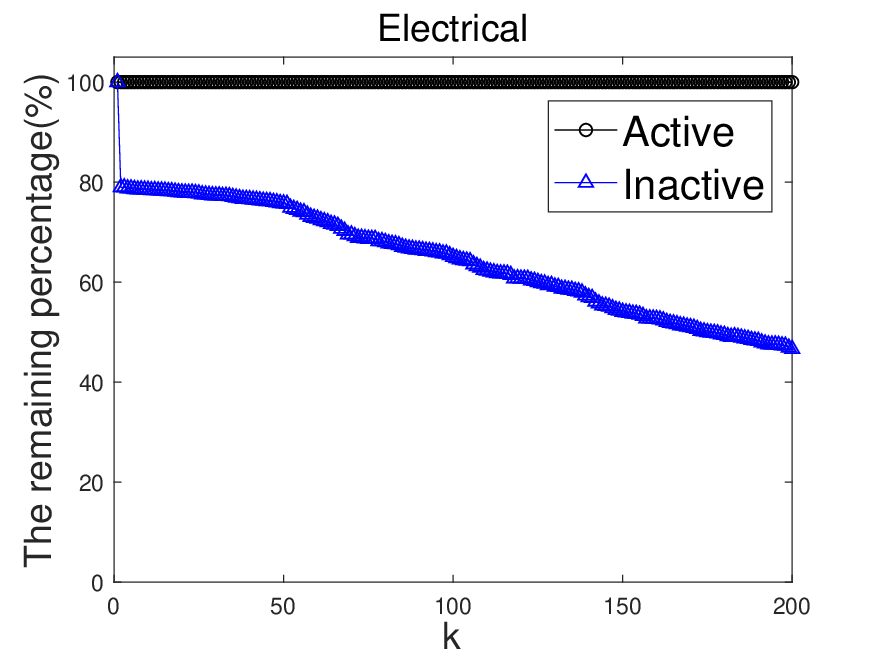}}
		\subfloat[]
		{\includegraphics[width=0.25\textwidth]{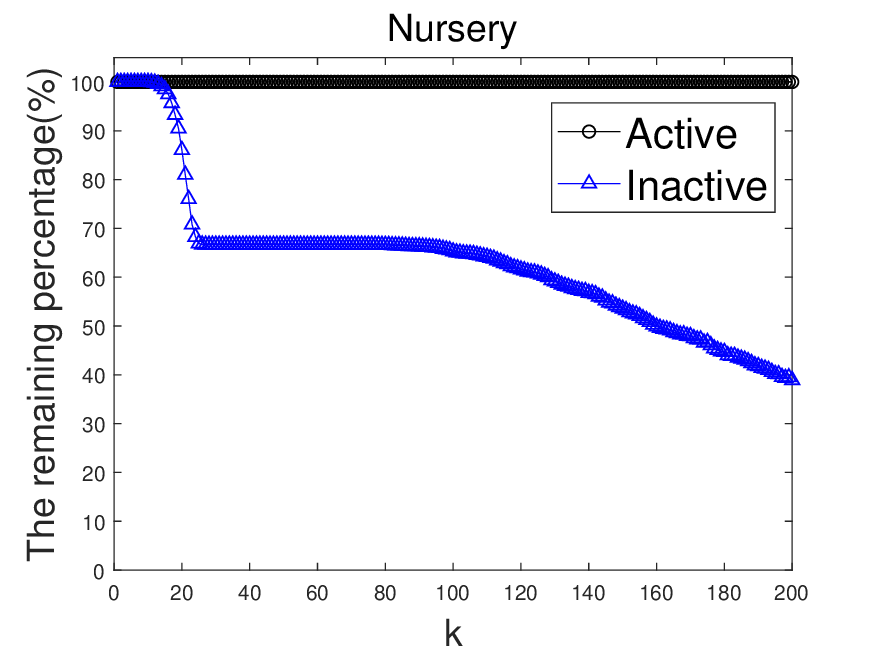}}
		
		\subfloat[]
		{\includegraphics[width=0.25\textwidth]{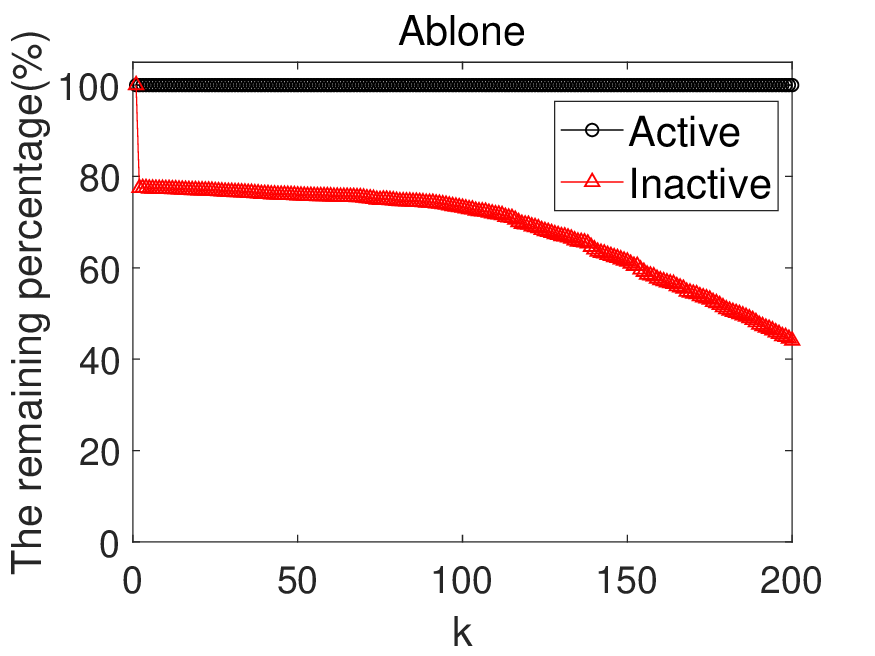}}
		\subfloat[]
		{\includegraphics[width=0.25\textwidth]{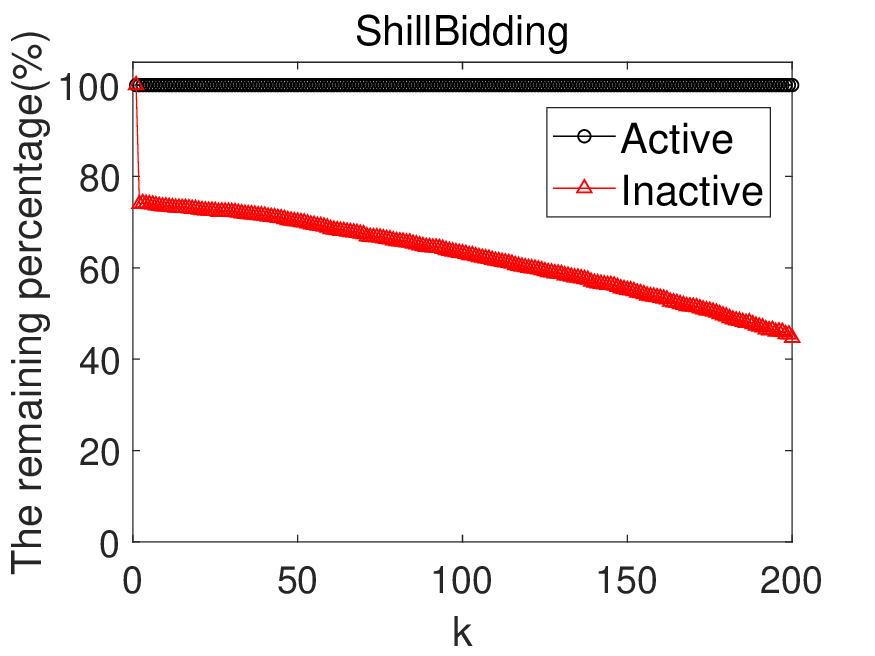}}
		\subfloat[]
		{\includegraphics[width=0.25\textwidth]{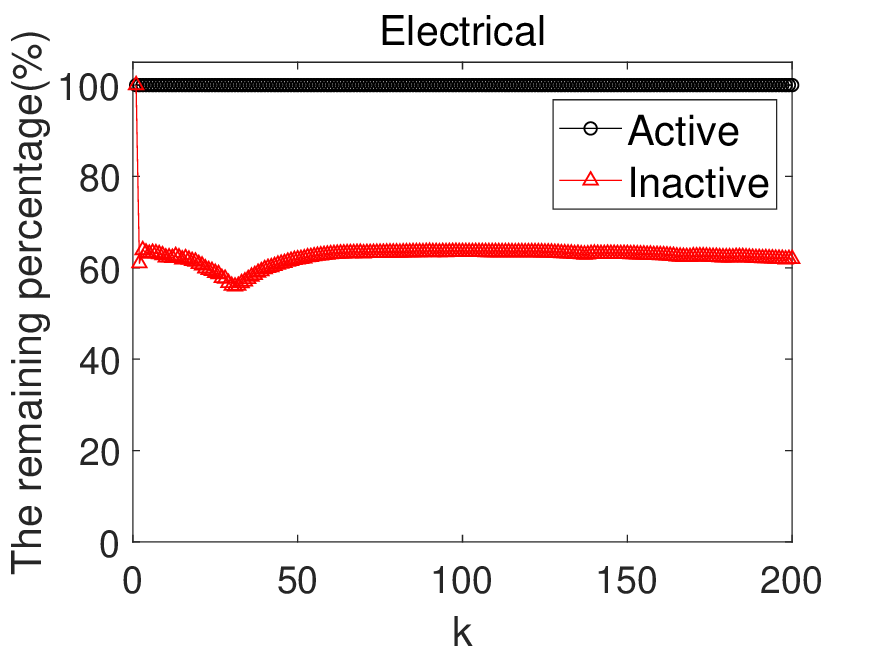}}
		\subfloat[]
		{\includegraphics[width=0.25\textwidth]{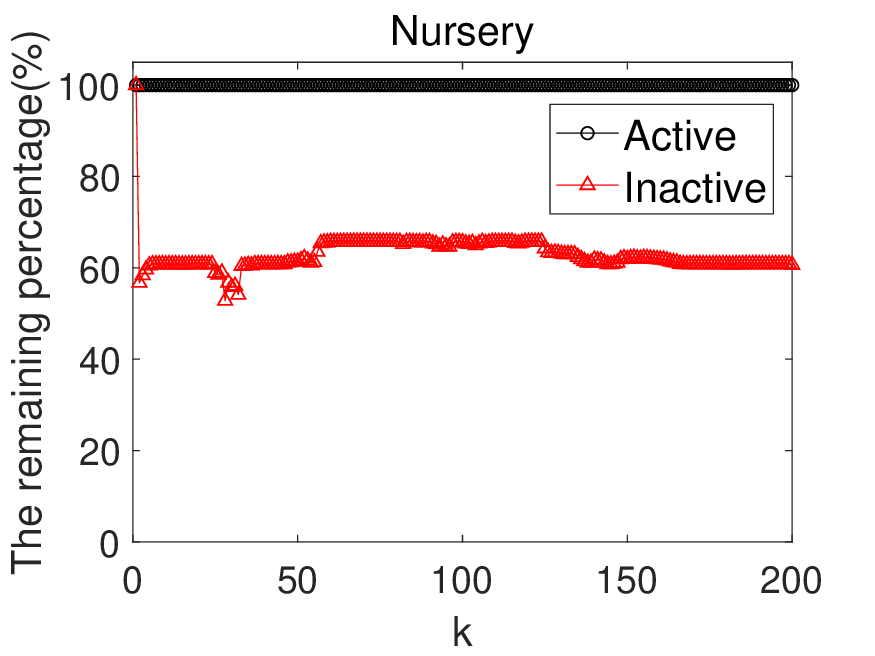}}
		\caption{The changing curves of remaining instances by the screening rule with different values of $\nu$. The horizontal axis denotes the $k$-th value of $\nu$ from the range of ($0.01:0.001:1-1/l$). The vertical axis represents the percentage of remaining instances after screening process. The first row shows linear results, and the second row shows nonlinear results. }
		\label{fig5}
  	  \vspace{-1.5em}
	\end{figure*}
	
	Fig. \ref{fig5} shows the remaining percentages of active and inactive samples at each value of the parameter $\nu$ for our SRBO. We randomly select four data sets to observe the changing curves with different parameter values. We found that most inactive instances can be deleted by our proposed SRBO-$\nu$-SVM, and all active instances remain in training, which demonstrates the safety of our proposed SRBO-$\nu$-SVM.
	
	\begin{figure*}
		\centering
		\subfloat[]
		{\includegraphics[width=0.18\textwidth]{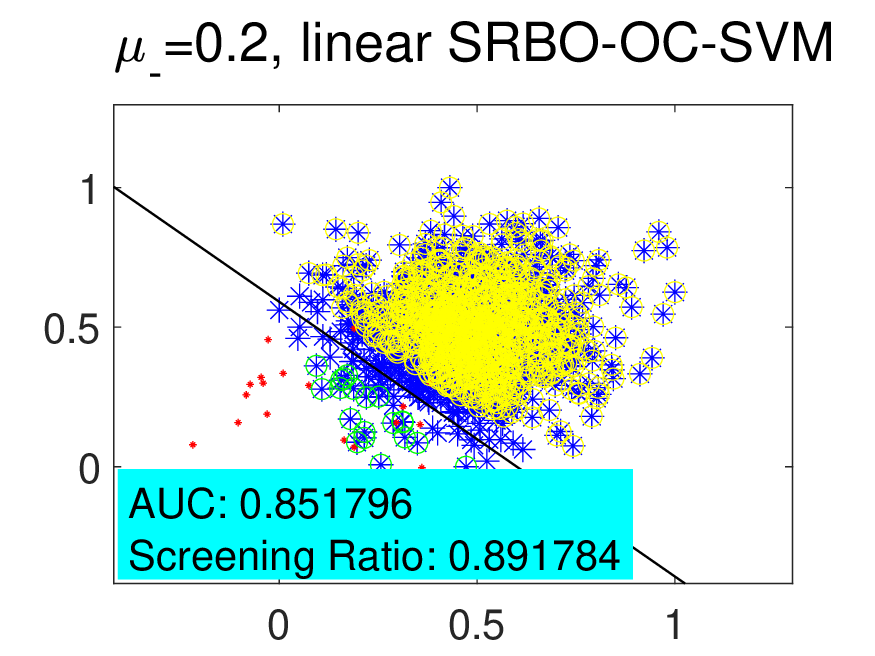}}
		\subfloat[]
		{\includegraphics[width=0.18\textwidth]{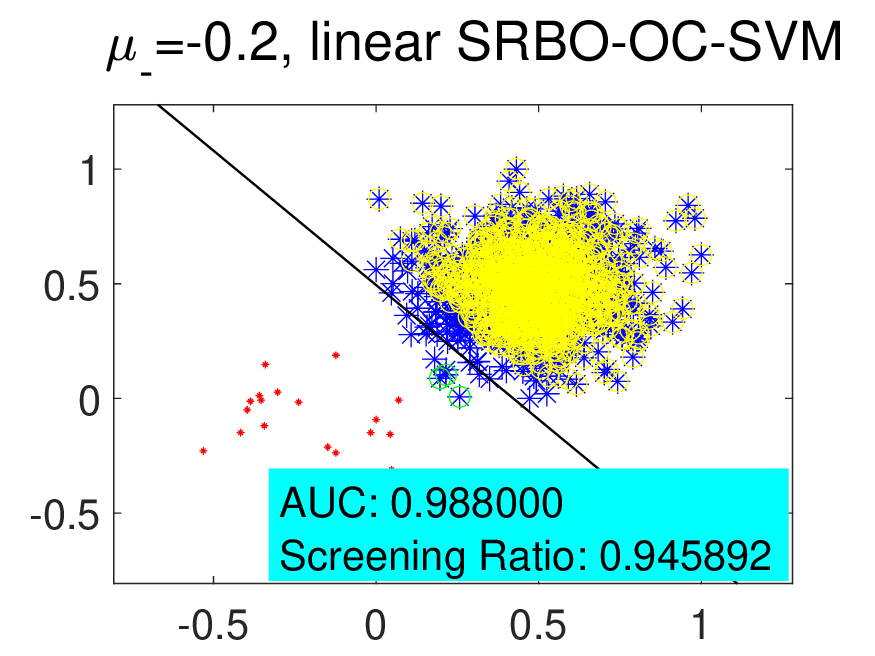}}
		\subfloat[]
		{\includegraphics[width=0.18\textwidth]{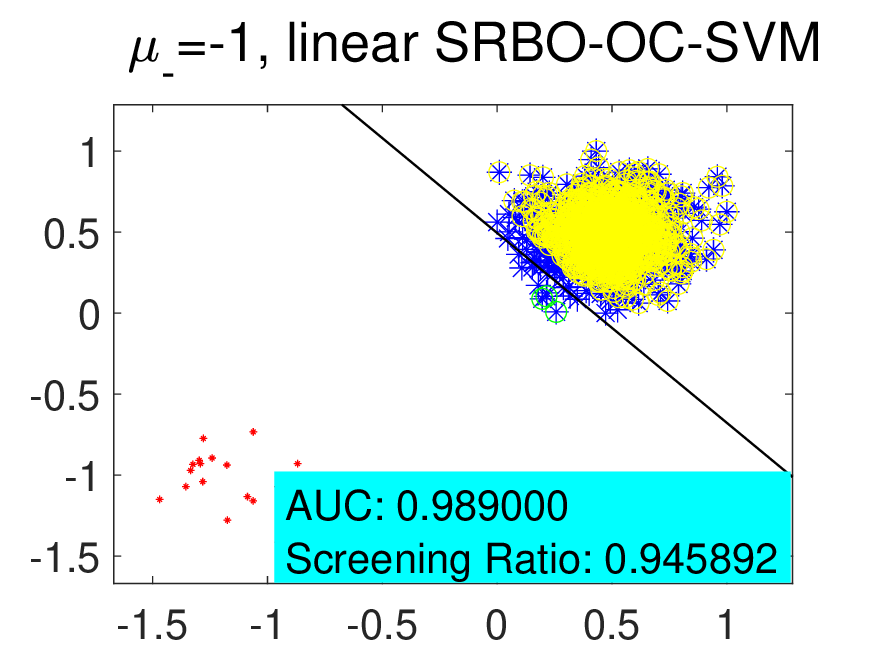}} 	    \subfloat[]
		{\includegraphics[width=0.18\textwidth]{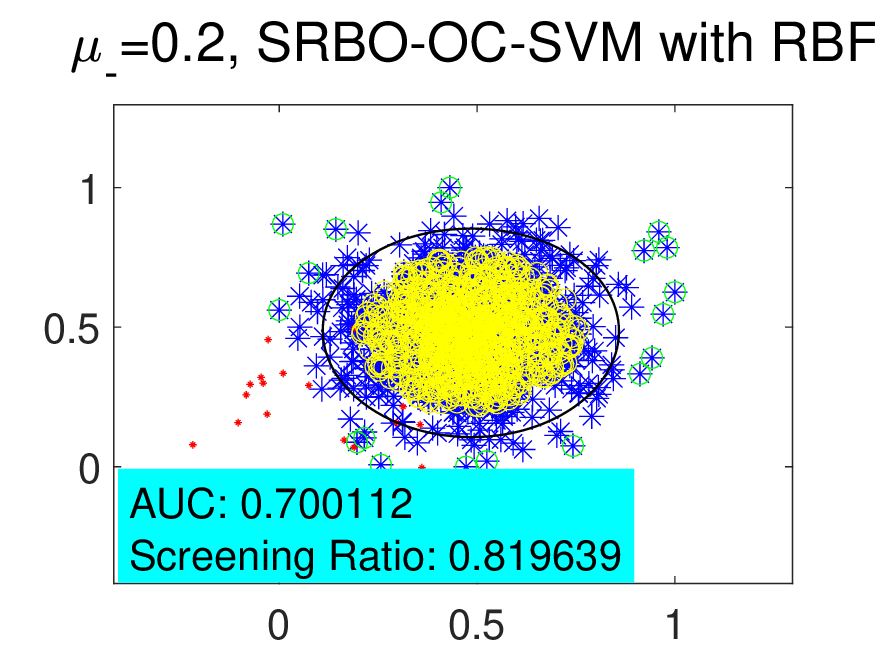}}
		\subfloat[]
		{\includegraphics[width=0.18\textwidth]{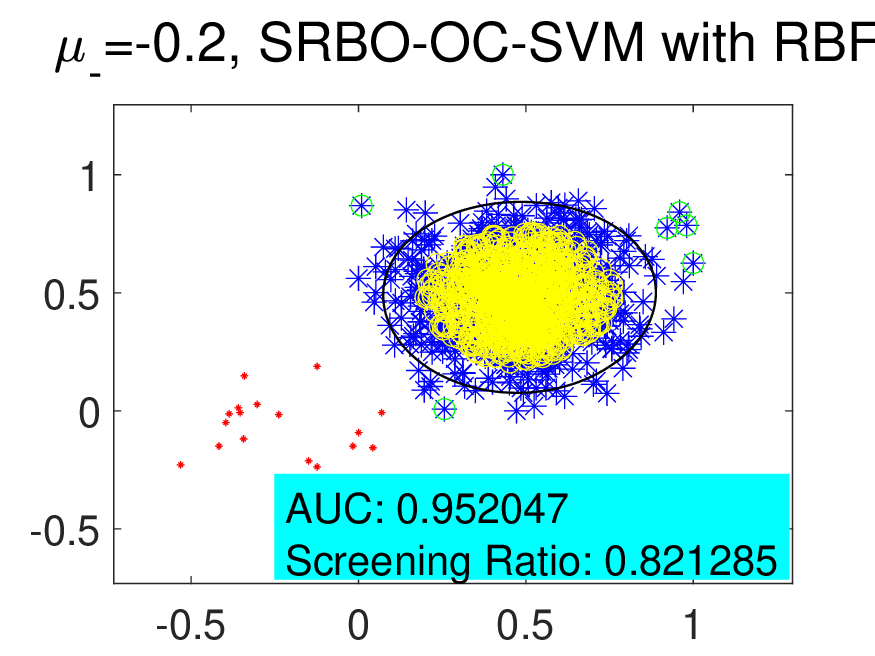}}
		
		\subfloat[]
		{\includegraphics[width=0.18\textwidth]{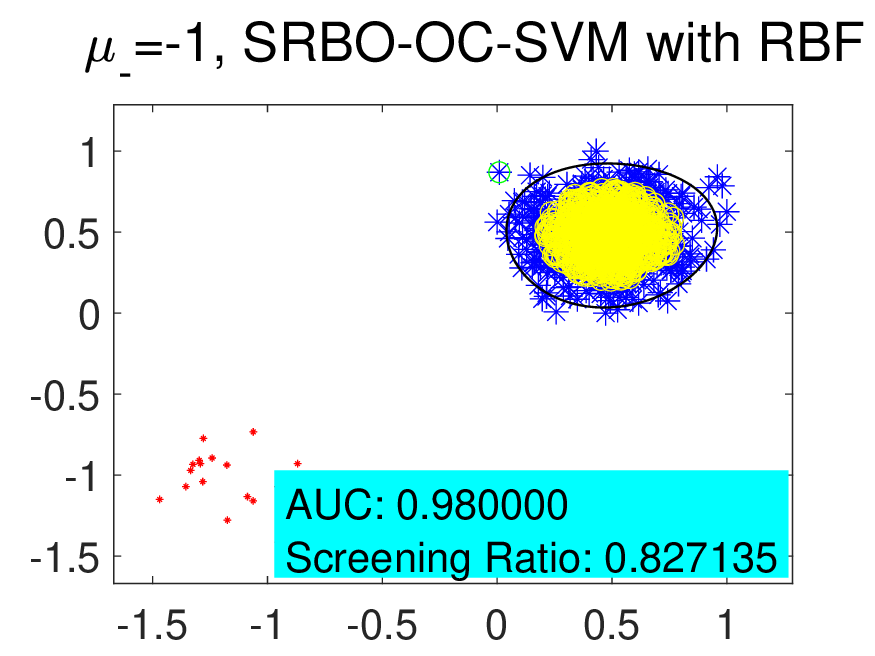}}
		\subfloat[]
		{\includegraphics[width=0.18\textwidth]{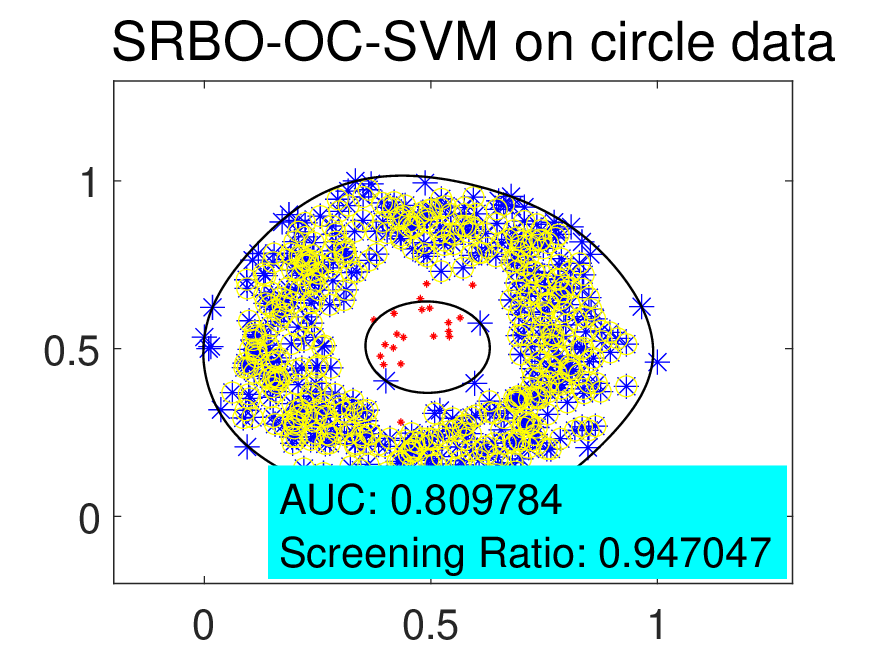}}
		\subfloat[]
		{\includegraphics[width=0.18\textwidth]{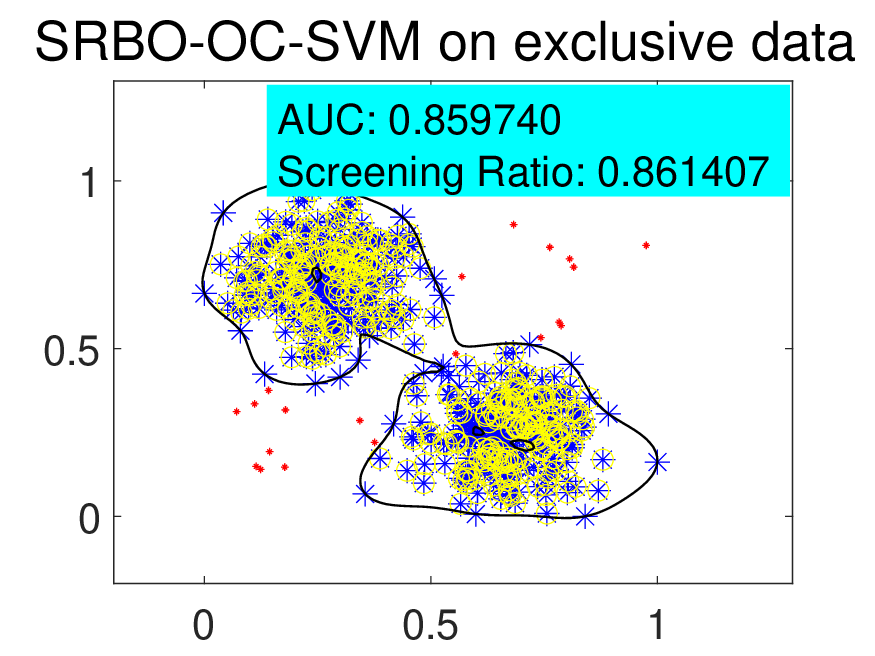}}
		\subfloat[]
		{\includegraphics[width=0.18\textwidth]{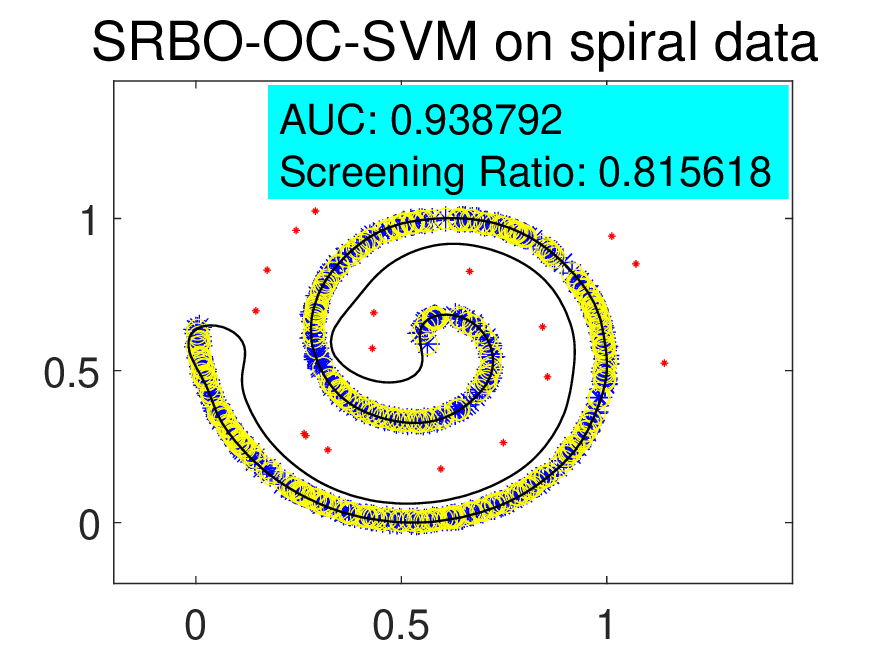}}
		\caption{Classification graphs of SRBO-OC-SVM on three normally distributed data sets (respectively in linear and nonlinear case, all of them with $\mu_{+}=0.5$, but $\mu_{-}=0.2,-0.2,-1$), nonlinear case on circle, exclusive and spiral data. The blue and red points represent the normal and abnormal training instances, respectively. In each graph, the solid black line is the decision boundary. Each graph corresponds to the classifier under optimal parameters, and `AUC' corresponds to the testing results on all normal and abnormal data. In the second row, the green points correspond to the samples deleted in $\mathcal{L}$, and the yellow points correspond to the samples removed in $\mathcal{R}$. `Screening Ratio' is the average result during the whole parameter selection process by SRBO.} 
		\label{fig333}
        \vspace{-1.5em}
	\end{figure*}
	
	In general, our proposed SRBO-$\nu$-SVM could greatly accelerate $\nu$-SVM and guarantee the same predication accuracy as the original $\nu$-SVM.

	\subsection{Experiments on Unsupervised Models}  \label{Experiment: OC-SVM}
Similarly, we evaluated the proposed SRBO in the unsupervised OC-SVM model. 
	
	\begin{table*}[htbp!]
		\renewcommand\arraystretch{1.1}
        \centering
		\resizebox{0.8\textwidth}{!}{
		\begin{threeparttable}[b]
			\centering
			\caption{Comparisons of 3 kinds of unsupervised methods on 26 benchmark data sets in linear case.}\label{Table5}
			\begin{tabular}{*{9}{c}}
				\toprule
				\multirow{2}{*}{Data set} &
				\multicolumn{2}{c}{KDE}&\multicolumn{2}{c}{OC-SVM}&\multicolumn{4}{c}{SRBO-OC-SVM}\\
				\cmidrule(lr){2-3}\cmidrule(lr){4-5}\cmidrule(lr){6-9}
				&AUC(\%)&Time(s)&AUC(\%)&Time(s)&AUC(\%)&Time(s)&Screening Ratio(\%)&Speedup Ratio \\
				\midrule
				Hepatitis
				&\bf{93.33}&0.2008&80.00&\bf{0.0090}&80.00&0.0111&68.70&0.8076\\
				Fertility
				&55.00&0.2468&\bf{90.00}&\bf{0.0048}&\bf{90.00}&0.0119&43.42&0.4045\\
				Planning Relax
				&58.33&0.4576&\bf{72.22}&\bf{0.0050}&\bf{72.22}&0.0101&60.66&0.4960\\
				Sonar
				&\bf{76.19}&0.6694&52.38&\bf{0.0046}&52.38&0.0106&65.25&0.6525\\
				SpectHeart
				&\bf{77.78}&0.6679&\bf{77.78}&0.0099&\bf{77.78}&\bf{0.0098}&42.62&\bf{1.0196}\\
				Haberman
				&49.13&1.0277&\bf{73.77}&\bf{0.0093}&\bf{73.77}&0.0181&48.62&0.5129\\
				LiverDisorder
				&59.42&1.4737&\bf{63.77}&\bf{0.0089}&\bf{63.77}&0.0177&49.09&0.4997\\
				Monks
				&\bf{59.77}&1.4005&54.02&\bf{0.0094}&54.02&0.0132&57.65&0.7150\\
				BreastCancer569
				&40.71&1.6861&\bf{61.95}&0.0182&\bf{61.95}&\bf{0.0147}&67.00&\bf{1.1242}\\
				BreastCancer683
				&\bf{89.71}&2.1590&63.97&0.0242&63.97&\bf{0.0187}&27.88&\bf{1.2894}\\
				Australian
				&\bf{76.98}&2.0877&67.63&0.0274&67.63&\bf{0.0196}&38.88&\bf{1.3974}\\
				Pima
				&62.75&2.3810&\bf{65.36}&0.0405&\bf{65.36}&\bf{0.0182}&40.51&\bf{2.2253}\\
				Biodegration
				&\bf{73.93}&3.1987&65.40&0.0729&65.40&\bf{0.0334}&54.60&\bf{2.1854}\\
				Banknote
				&67.64&4.9516&\bf{96.36}&0.0720&\bf{96.36}&\bf{0.0492}&41.28&\bf{1.4630}\\
				HCV-Egy
				&57.76&4.2513&\bf{73.65}&0.1707&\bf{73.65}&\bf{0.0351}&67.77&\bf{4.8605}\\
				CMC
				&55.10&4.9801&\bf{57.14}&0.1067&\bf{57.14}&\bf{0.0214}&74.13&\bf{4.9893}\\
				Yeast
				&55.89&4.6836&\bf{69.02}&0.1868&\bf{69.02}&\bf{0.0474}&59.81&\bf{3.9442}\\
				Wifi-localization
				&\bf{84.50}&7.0754&75.00&0.3769&75.00&\bf{0.2174}&42.55&\bf{1.7339}\\
				CTG
				&\bf{93.41}&7.3492&76.71&0.5642&76.71&\bf{0.0735}&72.90&\bf{7.6720}\\
				Abalone
				&71.38&18.5725&\bf{83.48}&2.9363&\bf{83.48}&\bf{0.8583}&56.07&\bf{3.4210}\\
				Winequality
				&\bf{77.24}&22.9347&75.41&5.0440&75.41&\bf{0.8254}&35.27&\bf{6.1109}\\
				ShillBidding
				&\bf{89.88}&36.2951&88.77&13.8901&88.77&\bf{3.1883}&57.02&\bf{4.3566}\\
				Musk
				&\bf{94.62}&57.5992&83.70&13.4376&83.70&\bf{4.4493}&50.16&\bf{3.0202}\\
				Electrical
				&61.60&99.7598&\bf{62.85}&16.7832&\bf{62.85}&\bf{6.6371}&41.17&\bf{2.5287}\\  
				Epiletic
				&66.30&204.8250&\bf{77.39}&59.6642&\bf{77.39}&\bf{2.9646}&32.82&\bf{20.1253}\\
				Nursery
				&33.49&186.5547&\bf{72.45}&42.9865&\bf{72.45}&\bf{1.6138}&78.85&\bf{26.6365}\\
                \midrule
   	        Win/Draw/Loss&14/1/11&&0/26/0&&&\\
                Win/Draw/Loss&&26/0/0&&19/0/7&\\
				\bottomrule
			\end{tabular}
		\end{threeparttable}}
	\end{table*}
	
	\begin{table*}[htbp!]
		\renewcommand\arraystretch{1.1}
       \centering
		\resizebox{0.8\textwidth}{!}{
		\begin{threeparttable}[b]
			\centering
			\caption{Comparisons of 3 kinds of unsupervised methods on 26 benchmark data sets in nonlinear case.}\label{Table6}
			\begin{tabular}{*{9}{c}}
				\toprule
				\multirow{2}{*}{Data set} &
				\multicolumn{2}{c}{KDE}&\multicolumn{2}{c}{OC-SVM}&\multicolumn{4}{c}{SRBO-OC-SVM}\\
				\cmidrule(lr){2-3}\cmidrule(lr){4-5}\cmidrule(lr){6-9}
				&AUC(\%)&Time(s)&AUC(\%)&Time(s)&AUC(\%)&Time(s)&Screening Ratio(\%)&Speedup Ratio \\
				\midrule
				Hepatitis
				&\bf{93.33}&0.2006&73.33&\bf{0.0069}&73.33&0.0078&75.26&0.8905\\
				Fertility
				&55.00&0.2538&\bf{85.00}&\bf{0.0047}&\bf{85.00}&0.0078&65.34&0.5980\\
				Planning Relax
				&58.33&0.4576&\bf{63.89}&\bf{0.0066}&\bf{63.89}&0.0083&91.34&0.7936\\
				Sonar
				&\bf{76.19}&0.5052&50.00&\bf{0.0068}&50.00&0.0110&50.86&0.6169\\
				SpectHeart
				&\bf{77.78}&0.6679&75.93&\bf{0.0116}&75.93&0.0138&34.10&0.8443\\
				Haberman
				&49.18&0.9566&\bf{60.66}&0.0080&\bf{60.66}&\bf{0.0059}&65.90&\bf{1.3555}\\
				LiverDisorder
				&59.42&0.9911&\bf{69.57}&\bf{0.0095}&\bf{69.57}&0.0118&62.32&0.8004\\
				Monks
				&\bf{59.77}&1.2957&58.62&\bf{0.0120}&58.62&0.0159&73.61&0.7560\\
				BreastCancer569
				&40.71&1.5674&\bf{93.81}&0.0196&\bf{93.81}&\bf{0.0166}&60.08&\bf{1.1831}\\
				BreastCancer683
				&89.71&2.0082&\bf{90.44}&0.0274&\bf{90.44}&\bf{0.0158}&60.65&\bf{1.7375}\\
				Australian
				&\bf{76.98}&2.0389&54.68&0.0311&54.68&\bf{0.0273}&64.25&\bf{1.1389}\\
				Pima
				&62.75&2.2388&\bf{68.63}&0.0294&\bf{68.63}&\bf{0.0116}&41.25&\bf{2.5245}\\
				Biodegration
				&\bf{73.93}&2.7417&65.41&0.0650&65.41&\bf{0.0297}&62.60&\bf{2.1917}\\
				Banknote
				&67.64&4.6092&\bf{70.55}&0.0878&\bf{70.55}&\bf{0.0297}&33.66&\bf{2.9559}\\
				HCV-Egy
				&57.76&4.2379&\bf{74.37}&0.1440&\bf{74.37}&\bf{0.0328}&32.94&\bf{4.3928}\\
				CMC
				&55.10&4.8932&\bf{57.48}&0.0993&\bf{57.48}&\bf{0.0349}&50.78&\bf{2.8484}\\
				Yeast
				&55.89&4.4471&\bf{69.70}&0.1999&\bf{69.70}&\bf{0.0370}&62.07&\bf{5.4090}\\
				Wifi-localization
				&\bf{84.50}&7.3994&74.50&0.3534&74.50&\bf{0.0661}&63.42&\bf{5.3466}\\
				CTG
				&93.41&7.1585&\bf{93.65}&0.4370&\bf{93.65}&\bf{0.0905}&61.18&\bf{4.8301}\\
				Abalone
				&71.38&18.7118&\bf{83.47}&3.2294&\bf{83.47}&\bf{1.2252}&47.31&\bf{2.6357}\\
				Winequality
				&\bf{77.24}&22.4314&72.35&3.9918&72.35&\bf{0.4204}&30.85&\bf{9.4940}\\
				ShillBidding
				&89.88&38.9719&\bf{95.18}&17.1932&\bf{95.18}&3.6904&52.15&\bf{4.6589}\\
				Musk
				&\bf{94.62}&56.6843&83.17&14.6424&83.17&\bf{4.0065}&49.93&\bf{3.6547}\\
				Electrical
				&\bf{61.60}&82.9467&57.85&17.1411&57.85&\bf{0.7888}&39.39&\bf{21.7307}\\  
				Epiletic
				&66.30&123.4692&\bf{72.13}&42.0975&\bf{72.13}&\bf{4.6391}&32.22&\bf{9.0745}\\
				Nursery
				&33.49&133.4568&\bf{98.34}&41.1782&\bf{98.34}&\bf{5.3014}&66.77&\bf{7.7675}\\
                \midrule
   	        Win/Draw/Loss&16/0/10&&0/26/0&&&\\
                Win/Draw/Loss&&26/0/0&&19/0/7&\\
				\bottomrule
			\end{tabular}
		\end{threeparttable}}
	\end{table*}

	\textbf{Experiments on 6 Artificial Data Sets. }  \label{OC-SVM: Artifical}
The advantages of our proposed SRBO-OC-SVM are verified on 6 artificial data sets. The data are generated in a manner similar to the supervised case. Considering that OC-SVM is used for anomaly detection, the negative sample size in each data set is reduced to 20\% of its original size. Positive samples are used to train the model and then both positive and negative samples are used to evaluate the model during testing. Taking into account the imbalance of positive and negative classes, AUC values are used to evaluate the prediction performance. The classification graphs are drawn in Fig. \ref{fig333}, respectively. 
	
	In Fig. \ref{fig333}, according to the `Screening Ratio', our SRBO-OC-SVM could screen out most inactive instances, which implies the effectiveness of our SRBO-OC-SVM. In addition, the decision hyperplanes and the results of `AUC' of our SRBO-OC-SVM are exactly the same as the original OC-SVM, which demonstrates the safety of SRBO-OC-SVM.
	
	\textbf{Experiments on 26 Benchmark Data Sets. }  \label{OC-SVM: benchmark}
	The performance comparisons of our SRBO-OC-SVM with the density kernel estimator (KDE) and the original OC-SVM are shown in Table \ref{Table5} and Table \ref{Table6}. 
	
	The time of SRBO-OC-SVM is significantly shorter than that of the original OC-SVM both in the linear and nonlinear cases. The larger the data sets, the more obvious acceleration effect, which verifies that our proposed safe screening rule with bi-level optimization achieves a superior performance for large-scale problem. By comparing the AUC of the three methods, OC-SVM is higher than KDE on most of the data sets, which shows that OC-SVM has better prediction ability. The AUC of OC-SVM and SRBO-OC-SVM is also consistent.
	
	In general, our SRBO has achieved excellent acceleration performance in both supervised $\nu$-SVM and unsupervised OC-SVM. It can greatly reduce computational time and guarantee safety.

	\subsection{Comparison of DCDM and `quadprog'}  \label{Experiment: comparison}
    	In this section, we compare the performance of two different solvers, i.e. our proposed DCDM and `quadprog' in the MATLAB toolbox. For a better comparison, we choose 5 medium-scale benchmark data sets (sample size greater than 10,000). Two metrics are use, i.e., `Accuracy(\%)' and `Time(s)'. `Accuracy(\%)' represents the prediction accuracy under optimal parameter in each case. The `Time(s)' refers to complete one full experiment for each model with the optimal parameter. Concretely, the `Time(s)' for $\nu$-SVM represents the computational time to solve the QPP with optimal parameters. For SRBO-$\nu$-SVM, it is the sum of three parts: (1) The time to solve the hidden variable $\bm{\delta}$. (2) The time for safe screening. (3) The time to solve the reduced small-scale problem after screening.

      As shown in Figs. \ref{bar_com_linear} and \ref{bar_com_rbf}, for each data set, the first two are significantly higher than the latter two. This indicates that the running time of DCDM is much shorter than that of `quadprog', and the difference is even nearly fifty times. Our proposed screening rule shows good acceleration in both solvers. Since the computational time of `quadprog' on `Adult' data set is very long and is not comparable to DCDM, we do not provide its results.

      Table \ref{Table_comparisons} provides the corresponding accuracy results. In each case, the use of SRBO does not affect the accuracy. This verifies the safety of our SRBO-$\nu$-SVM. 

\begin{table*}[htbp]
		\renewcommand\arraystretch{1.1}
        \centering
	  \resizebox{0.6\textwidth}{!}{
			\begin{threeparttable}[b]
				\centering
				\caption{Accuracy comparison of different algorithms of $\nu$-SVM on 5 medium-scale benchmark data sets.}\label{Table_comparisons}
				\begin{tabular}{*{9}{c}}
					\toprule
					\multirow{2}{*}{Data set} &\multicolumn{4}{c}{$\nu$-SVM}&\multicolumn{4}{c}{SRBO-$\nu$-SVM}\\
					\cmidrule(lr){2-5}\cmidrule(lr){6-9}
					&\multicolumn{2}{c}{quadprog}&\multicolumn{2}{c}{DCDM}&\multicolumn{2}{c}{quadprog}&\multicolumn{2}{c}{DCDM}\\
					\cmidrule(lr){2-3}\cmidrule(lr){4-5}\cmidrule(lr){6-7}\cmidrule(lr){8-9}
					&Linear&RBF&Linear&RBF&Linear&RBF&Linear&RBF \\
					\midrule			
					Epiletic
					&\bf{81.45}&\bf{96.70}&80.83&96.35&\bf{81.45}&\bf{96.70}&80.83&96.35\\
					Nursery
					&\bf{100.00}&\bf{100.00}&91.05&\bf{100.00}&\bf{100.00}&\bf{100.00}&91.05&\bf{100.00}\\
					credit card
					&\bf{41.47}&\bf{41.47}&\bf{41.47}&\bf{41.47}&\bf{41.47}&\bf{41.47}&\bf{41.47}&\bf{41.47}\\
					Accelerometer
					&\bf{100.00}&\bf{100.00}&\bf{100.00}&\bf{100.00}&\bf{100.00}&\bf{100.00}&\bf{100.00}&\bf{100.00}\\
					Adult
					&-&-&\bf{92.75}&\bf{76.92}&-&-&\bf{92.75}&\bf{76.92}\\
                   \bottomrule
				\end{tabular}
		\end{threeparttable}}
	\end{table*}
 
    \begin{figure*}
		\centering
		\subfloat[Linear kernel]
		{    \label{bar_com_linear}
			\includegraphics[width=0.49\textwidth]{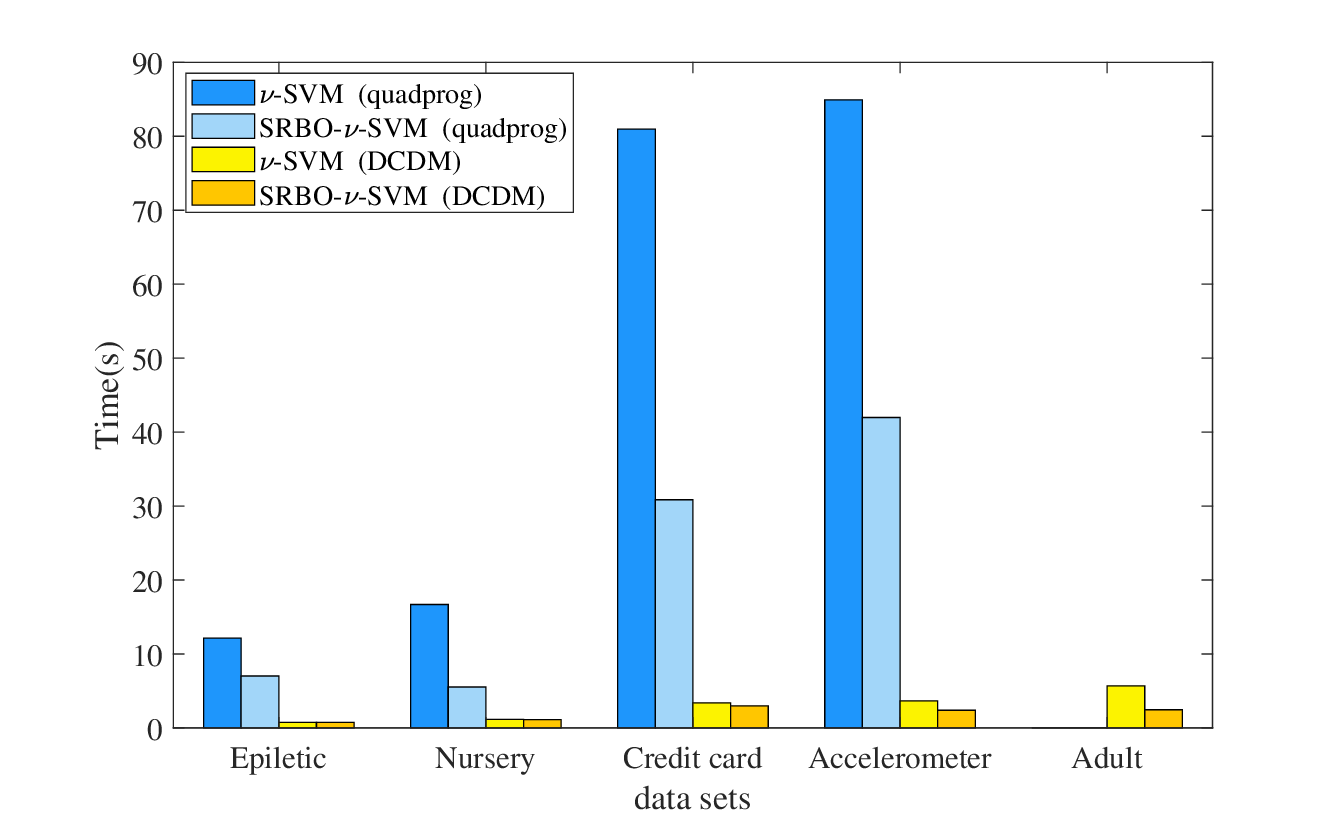}}
		\subfloat[RBF kernel]
		{ \label{bar_com_rbf}
			\includegraphics[width=0.49\textwidth]{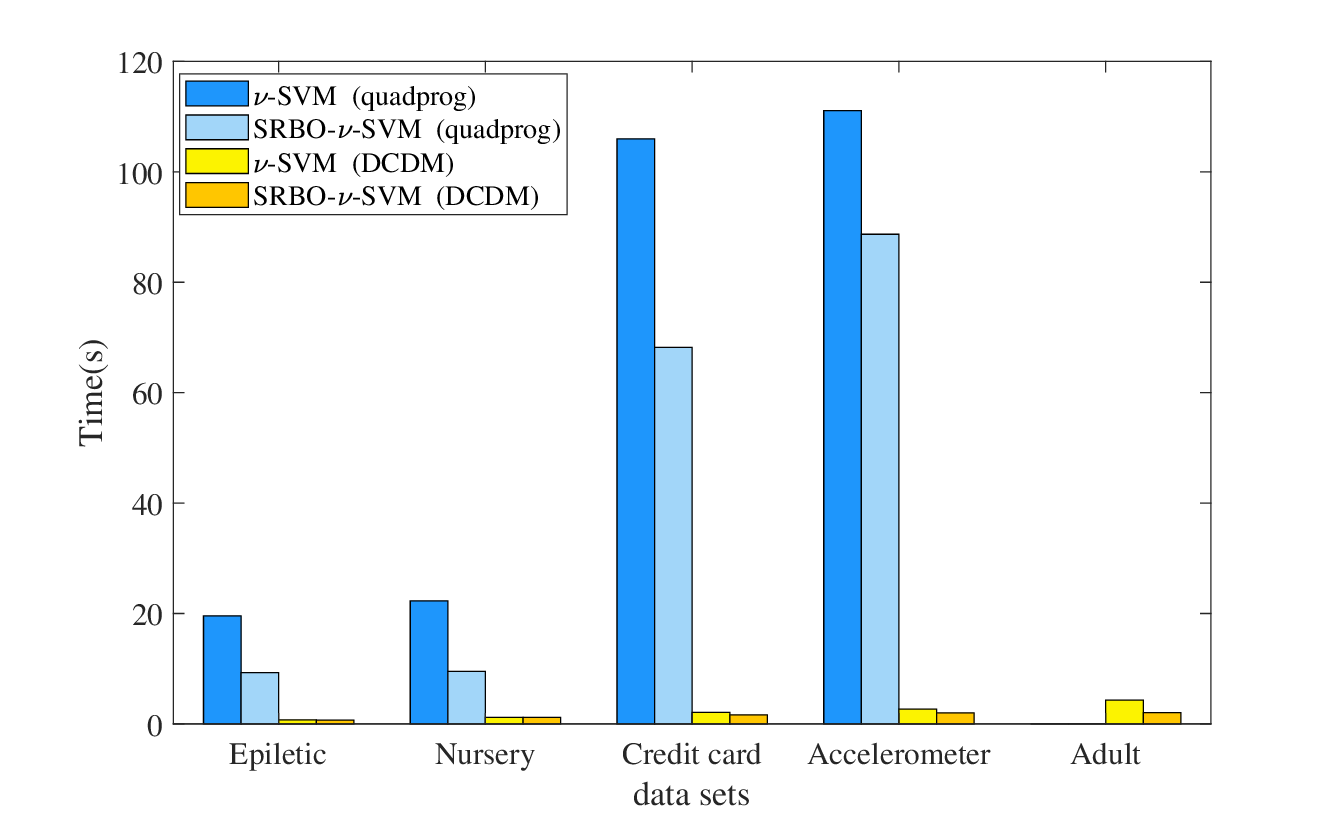}}
		\caption{Time comparison between quadprog and DCDM with linear and RBF kernel on large-scale benchmark data sets.}
		\label{com}
	\end{figure*}

	\subsection{MNIST Data Set}
	To further verify the efficiency of our method on a large-scale and high-dimensional data set, we performed experiments on the MNIST image data set. It is a common benchmark data set for the recognition of handwritten digits (0 through 9). It consists of 60000 training samples and 10000 test samples. Each sample is a handwritten grayscale digit image with $28\times28$ pixels. We have transformed each image into a vector with dimensions $782$. The sample sizes for each category are shown in Table \ref{TableMNIST}. We treated class 1 as a positive class sample and the other classes as negative classes, respectively. The results are shown in Tables \ref{MNIST_linear} and \ref{MNIST_nonlinear}. 
	
	From the time point of view, in almost each case, our proposed SRBO has accelerated the original model. Moreover, the time of DCDM is significantly shorter than that of `quadprog'. The speed of all the methods is faster in the nonlinear case, which is due to the fact that nonlinear models are more suitable for dealing with high-dimensional data. In addition, the accuracies of the methods with and without SRBO are exactly the same. This verifies the safety of our SRBO-$\nu$-SVM. 

  Notice that the screening ratio and speedup ratio of our SRBO on high-dimensional data are not as prominent as that of the previous low-dimensional data. This is caused by the complex structure of the data itself. Exploring how the increased dimensionality of data affect the performance of safe screening is one of our future work. 
	
	\begin{table}[htbp!]
		\centering
		\caption{Sample sizes for 10 categories of MNIST data set.}\label{TableMNIST}
       \resizebox{0.6\textwidth}{!}{
		\begin{tabular}{ccccccccccc}\bottomrule            
    Label & 0 & 1 & 2 & 3 & 4 & 5 & 6 & 7 & 8 & 9\\
\bottomrule
$^\#$Training & 5923& 6742&5958 &6131 &5842 &5421 &5918 &6265 & 5851 & 5949 \\
$^\#$Test & 980 & 1135 & 1032 & 1010 & 982 & 892 & 958 & 1028 & 974 & 1009\\
			\bottomrule
		\end{tabular}}
	\end{table}
	
	\begin{table*}[htbp!]
		\renewcommand\arraystretch{1.1}
		\resizebox{\textwidth}{!}{
		\begin{threeparttable}[b]
			\centering
			\caption{Comparisons of $\nu$-SVM and SRBO-$\nu$-SVM on MNIST data set with positive class of digit ``1'' in linear case.}\label{MNIST_linear}
			\begin{tabular}{ccccccc|cccccc}
				\toprule
				\multirow{2}{*}{Negative Class}&\multicolumn{2}{c}{$\nu$-SVM (quadprog)}&\multicolumn{4}{c}{SRBO-$\nu$-SVM (quadprog)}&\multicolumn{2}{c}{$\nu$-SVM (DCDM)}&\multicolumn{4}{c}{SRBO-$\nu$-SVM ((DCDM))}\\
				\cmidrule(lr){2-3}\cmidrule(lr){4-7}\cmidrule(lr){8-9}\cmidrule(lr){10-13}
				&Accuracy(\%)&Time(s)&Accuracy(\%)&Time(s)&Screening Ratio(\%)&Speedup Ratio&Accuracy(\%)&Time(s)&Accuracy(\%)&Time(s)&Screening Ratio(\%)&Speedup Ratio\\
				\midrule
				0
				&\bf{99.60}&680.9950&\bf{99.60}&676.3948&16.48&\bf{1.0068}&99.05&419.6875&99.05&384.6198&16.48&\bf{1.0911}\\
				2
				&\bf{98.99}&338.5352&\bf{98.99}&333.4777&20.34&\bf{1.0151}&97.68&264.9781&97.68&203.6748&20.34&\bf{1.3009}\\
				3
				&\bf{98.83}&415.7089&9\bf{8.83}&313.5044&35.34&\bf{1.3260}&98.62&348.3195&98.62&184.6195&35.34&\bf{1.8866}\\
				4
				&99.51&464.0792&99.51&240.8822&42.16&\bf{1.9265}&\bf{99.69}&264.8426&\bf{99.69}&203.6485&42.16&\bf{1.3004}\\
				5
				&\bf{99.09}&345.6233&\bf{99.09}&154.6874&45.31&\bf{2.2343}&98.33&213.6482&98.33&168.1644&45.31&\bf{1.2704}\\
				6
				&\bf{99.51}&470.0865&\bf{99.51}&285.6789&31.26&\bf{1.6455}&99.31&224.6183&99.31&167.3495&31.26&\bf{1.3422}\\
				7
				&99.14&522.9732&99.14&506.3948&12.34&\bf{1.0327}&\bf{99.68}&469.3165&\bf{99.68}&326.1978&12.34&\bf{1.4387}\\
				8
				&98.18&247.8925&98.18&164.3978&30.16&\bf{1.5078}&\bf{98.84}&209.1648&\bf{98.84}&111.3462&30.16&\bf{1.8785}\\
				9
				&\bf{99.35}&474.3416&\bf{99.35}&297.3648&26.17&\bf{1.5951}&97.63&316.3498&97.63&216.3497&26.17&\bf{1.4622}\\
				\bottomrule
			\end{tabular}
		\end{threeparttable}}
	\end{table*}

	\begin{table*}[htbp!]
		\renewcommand\arraystretch{1.1}
		\resizebox{\textwidth}{!}{
		\begin{threeparttable}[b]
			\centering
			\caption{Comparisons of $\nu$-SVM and SRBO-$\nu$-SVM on MNIST data set with positive class of digit ``1'' in nonlinear case.}\label{MNIST_nonlinear}
			\begin{tabular}{ccccccc|cccccc}
				\toprule
				\multirow{2}{*}{Negative Class}&\multicolumn{2}{c}{$\nu$-SVM (quadprog)}&\multicolumn{4}{c}{SRBO-$\nu$-SVM (quadprog)}&\multicolumn{2}{c}{$\nu$-SVM (DCDM)}&\multicolumn{4}{c}{SRBO-$\nu$-SVM (DCDM)}\\
				\cmidrule(lr){2-3}\cmidrule(lr){4-7}\cmidrule(lr){8-9}\cmidrule(lr){10-13}
				&Accuracy(\%)&Time(s)&Accuracy(\%)&Time(s)&Screening Ratio(\%)&Speedup Ratio&Accuracy(\%)&Time(s)&Accuracy(\%)&Time(s)&Screening Ratio(\%)&Speedup Ratio\\
                \midrule
				0
				&\bf{100.00}&17.5818&\bf{100.00}&15.1975&22.16&\bf{1.1568}&\bf{100.00}&13.9486&\bf{100.00}&13.4862&22.16&\bf{1.0342}\\
				2
				&\bf{100.00}&17.7246&\bf{100.00}&17.9639&10.34&0.9866&\bf{100.00}&15.3947&\bf{100.00}&15.6487&10.34&0.9837\\
				3
				&\bf{100.00}&18.4645&\bf{100.00}&7.9486&29.43&\bf{2.3229}&\bf{100.00}&14.3875&\bf{100.00}&7.3648&29.43&\bf{1.9535}\\
				4
				&\bf{100.00}&17.4361&\bf{100.00}&11.3497&21.46&\bf{1.5362}&\bf{100.00}&12.6487&\bf{100.00}&11.3948&21.46&\bf{1.1100}\\
				5
				&\bf{100.00}&15.8443&\bf{100.00}&11.9487&25.31&\bf{1.3260}&\bf{100.00}&11.2846&\bf{100.00}&6.3485&25.31&\bf{1.7775}\\
				6
				&\bf{100.00}&17.7388&\bf{100.00}&9.3486&42.16&\bf{1.8974}&\bf{100.00}&11.3648&\bf{100.00}&8.3691&42.16&\bf{1.3579}\\
				7
				&\bf{100.00}&18.8371&\bf{100.00}&6.1843&37.84&\bf{3.0459}&\bf{100.00}&9.3648&\bf{100.00}&7.9154&37.84&\bf{1.1831}\\
				8
				&\bf{100.00}&17.5379&\bf{100.00}&12.3648&36.27&\bf{1.4183}&\bf{100.00}&13.6489&\bf{100.00}&6.3948&36.27&\bf{2.1343}\\
				9
				&\bf{100.00}&17.6672&\bf{100.00}&18.0033&14.35&\bf{0.9813}&\bf{100.00}&16.4684&\bf{100.00}&12.9486&14.35&\bf{1.2718}\\
				\bottomrule
			\end{tabular}
		\end{threeparttable}}
	\end{table*}
	
		\subsection{Wilcoxon Signed Rank Test}
	
	From the experimental results above, we observed that our proposed SRBO method does not outperform the original algorithms for a few data sets in terms of time. Here, we give a significance test analysis using the Wilcoxon signed rank test to demonstrate the efficiency of our proposed algorithms. 
	
	The original hypothesis $H_{0}$ and checking hypothesis $H_{1}$ are
	\begin{eqnarray}\label{H}
		H_{0}:M_{0} \leq  M_{1} ~~~ v.s ~~~  H_{1}:M_{0} >  M_{1},
	\end{eqnarray}
    where $M_{0}$ is the median runtime of the original SVMs, and $M_{1}$ is the median runtime of our proposed SRBO algorithm. If $H_{0}$ is rejected, it means that our proposed method is more efficient than the original method.
    
    When the test sample size $n$ is large ($n>20$), the statistic $Z$ can be approximately normally distributed
    \begin{eqnarray}\label{Wilcoxon_test}
		Z = \frac{W^{+}-n(n+1)/4}{\sqrt{n(n+1)(2n+1)/24}} \sim N(0,1),
	\end{eqnarray}
	where $W^{+} = \sum_{j=1}^{n} R_{j}^{+}I(a_{j}>0)$, $a_{j} = time_{SVMs} - time_{SRBO-SVMs}$ and $R_{j}^{+}$ denotes the rank of $|a_{j}|$ in the sample with absolute value.

	The results of the Wilcoxon signed rank test are presented in Table \ref{test}. At the level of significance $\alpha = 0.05$, it appears that the $p$-values satisfy $p<\alpha$ in most cases. That is, the original hypothesis $H_{0}$ can be rejected in most cases. This means that the time of our SRBO-SVMs is significantly less than that of the original SVMs. Therefore, the computational advantage of our method is statistically significant.

\begin{table*}[htbp!]
	\renewcommand\arraystretch{1.1}
    \centering
	\resizebox{0.7\textwidth}{!}{
	\begin{threeparttable}[b]
		\centering
		\caption{Statistical test on time results of all the experiments}\label{test}
		\begin{tabular}{*{13}{c}}
			\toprule
			\multirow{3}{*}{}&\multicolumn{4}{c}{26 small-scale data sets}&\multicolumn{4}{c}{5 medium-scale data sets}&\multicolumn{4}{c}{MNIST data set}\\
			\cmidrule(lr){2-5}\cmidrule(lr){6-9}\cmidrule(lr){10-13}
			&\multicolumn{2}{c}{$\nu$-SVM}&\multicolumn{2}{c}{OC-SVM}&\multicolumn{2}{c}{quadprog}&\multicolumn{2}{c}{DCDM}&\multicolumn{2}{c}{quadprog}&\multicolumn{2}{c}{DCDM}\\
			\cmidrule(lr){2-3}\cmidrule(lr){4-5}\cmidrule(lr){6-7}\cmidrule(lr){8-9}\cmidrule(lr){10-11}\cmidrule(lr){12-13}
			&Linear&RBF&Linear&RBF&Linear&RBF&Linear&RBF&Linear&RBF&Linear&RBF \\
			\midrule
			$n$
			&13&26&26&26&4&4&5&5&9&9&9&9 \\
			$W^{+}$
			&20&39&47&38&0&0&0&0&0&3&0&1 \\
			$Z$
			&-&-3.46&-3.26&-3.49&-&-&-&-&-&-&-&-\\
			$p$
			&0.0402$^{*}$ &0.0007$^{*}$&0.0003$^{*}$&0.0002$^{*}$&0.125&0.125&0.0313$^{*}$&0.0313$^{*}$&0.0020$^{*}$&0.0098$^{*}$&0.0020$^{*}$&0.0039$^{*}$ \\
			\bottomrule
		\end{tabular}
  \begin{tablenotes}
          \footnotesize
           \item[1] The ``*'' indicates that the corresponding result is statistically significant with the level of $\alpha = 0.05.$
      \end{tablenotes}
	\end{threeparttable}}
 
\end{table*}
				
	\section{Conclusion}
	In this paper, a safe screening rule with bi-level optimization is proposed to reduce the computational cost of the original $\nu$-SVM without losing its accuracy. The main idea is to construct the upper and lower bounds of variables based on variational inequalities, KKT conditions and the $\nu$-property to estimate a region included optimal solutions of the optimization problem. Second, we extend this idea to a unified framework of SVM-type models and propose a safe screening rule with bi-level optimization for OC-SVM. The proposed SRBO-$\nu$-SVM and SRBO-OC-SVM can identify inactive samples before solving optimization problems, greatly reducing computational time and keeping the solution unchanged. In addition, we propose the DCDM algorithm to improve the solution speed. Finally, we conduct numerical experiments on artificial data sets and benchmark data sets to verify the safety and effectiveness of our SRBO and DCDM.
	
	There are three main factors that will affect the efficiency of our screening rule. The first is the introduced hidden vector $\bm{\delta}$. Inappropriate values of $\bm{\delta}$ may cause the feasible region to be too large, so that samples will be screened with low efficiency. To address this issue, we design a small-scale optimization problem to obtain the desired $\bm{\delta}$, i.e., QPP (\ref{deltaQPP}). The second is the value of parameter $\nu$, since our screening rule depends on the solution with the previous parameter $\nu$. We have given the experimental analysis, shown in Fig. \ref{fig5}. However, the theoretical relationship between the parameter interval and the screening ratio has not been proved so far. The third is the structure and dimensions of the data itself. In the experiments, we find that this fact has an impact on the performance of our screening rule. 
	
   The following topics will be our future work. First, study on the relationship between parameter intervals and screening ratio. Second, explore how the increased dimensionality of the data affects the performance of safe screening. Third, extend the safe screening rule to the field of deep learning.
			
	\appendices	
	
		\bibliographystyle{elsarticle-num}
	\bibliography{mybibfile}	  
\end{document}